\documentclass[11pt,letterpaper]{article}

\usepackage{amsmath,amsthm,amssymb,amscd,amstext,amsfonts}
\usepackage{color,verbatim,graphicx,fullpage,url}
\usepackage{xcolor}
\usepackage{algorithm}
\usepackage{algorithmic}
\usepackage{nicefrac}
\usepackage{fullpage,tikz,enumitem}
\usepackage{tikz-cd}
\usepackage{bm}
\usepackage{physics}
\usepackage{todonotes}
\usepackage{mathtools}
\usepackage{thm-restate}
\usepackage[numbers]{natbib}

\usepackage[breaklinks=true, linktocpage=true,]{hyperref}
\hypersetup{
    colorlinks = true,
    linkcolor={purple},
    urlcolor={purple},
    citecolor={blue!80!black}
}

\usepackage[nameinlink, noabbrev, capitalize]{cleveref}

\usepackage[utf8]{inputenc}
 
\Crefname{enumi}{Property}{Properties}

\theoremstyle{plain}

\newtheorem{thm}{Theorem}[section]
\newtheorem{theorem}[thm]{Theorem}

\crefname{atheorem}{Theorem}{Theorems}
\Crefname{atheorem}{Theorem}{Theorems}

\newtheorem{conjecture}[thm]{Conjecture}

\newtheorem{corollary}[thm]{Corollary}
\newtheorem{lemma}[thm]{Lemma}
\newtheorem{proposition}[thm]{Proposition}
\newtheorem{definition}[thm]{Definition}
\newtheorem{example}[thm]{Example}
\newtheorem{problem}[thm]{Problem}

\newtheorem{claim}[thm]{Claim}
\newtheorem*{claim*}{Claim}

\newtheorem{remark}[thm]{Remark}

\RenewCommandCopy{\theHtheorem}{\thetheorem}

\renewcommand\norm[1]{\left|\!\left|#1\right|\!\right|}

\renewcommand{\abs}[1]{|#1|}

\newcommand{\set}[1]{\{#1\}}
\newcommand{\Set}[1]{\left\{#1\right\}}

\newcommand{\inp}[2]{{\left\langle #1,#2 \right\rangle}}            % inner product

\DeclareMathOperator{\Ex}{\mathbb{E}}           % expected value 
\newcommand{\N}{\mathbb{N}}

\newcommand{\RR}{\mathbb{R}}
\newcommand{\bbS}{\mathbb{S}}

\newcommand{\bZ}{\mathbb{Z}}

\newcommand{\cF}{\mathcal F}
\newcommand{\cI}{\mathcal I}

\newcommand{\cX}{\mathcal X}

\newcommand{\bS}{\mathbf S}

\newcommand{\PG}{\textnormal{PG}}

\newcommand{\eps}{\varepsilon}

\usepackage{ifthen}
\newcommand{\agnorm}[2][]{
	\ifthenelse{\equal{#2}{}}{
		\widetilde{\gamma}_2^{#1}
	}{
		\widetilde{\gamma}_2^{#1}(#2)
	}
}

\DeclareFontFamily{U}{mathx}{}
\DeclareFontShape{U}{mathx}{m}{n}{<-> mathx10}{}
\DeclareSymbolFont{mathx}{U}{mathx}{m}{n}
\DeclareMathAccent{\widecheck}{0}{mathx}{"71}

\newcommand{\cC}{\mathcal{C}}
\newcommand{\cS}{\mathcal{S}}
\newcommand{\cJ}{\mathcal{J}}
\newcommand{\cH}{\mathcal{H}}
\newcommand{\cA}{\mathcal{A}}
\newcommand{\cK}{\mathcal{K}}
\newcommand{\cD}{\mathcal{D}}
\newcommand{\cE}{\mathcal{E}}
\newcommand{\cL}{\mathcal{L}}

\newcommand{\cP}{\mathcal{P}}
\newcommand{\cU}{\mathcal{U}}
\newcommand{\E}{\mathbb{E}}

\DeclareMathOperator{\VCdim}{\textnormal{\textsc{vc}}}

\DeclareMathOperator{\LR}{\textnormal{\textsc{lr}}}

\DeclareMathOperator{\tdim}{Tdim}
\DeclareMathOperator{\edim}{Edim}

\DeclareMathOperator{\sign}{sign}

\newcommand{\TV}{\mathtt{TV}}

\renewcommand{\cH}{\mathcal{H}}
\renewcommand{\bS}{\mathbb{S}}
\newcommand{\ZZ}{\mathbb{Z}}
\newcommand{\Z}{\mathbb{Z}}

\renewcommand{\Pr}{\mathbb{P}}

\newcommand{\height}{\operatorname{H}}

\DeclareMathOperator{\conv}{conv}                                             % convex hull
\DeclareMathOperator{\supp}{supp}                                             % support
\DeclareMathOperator{\loss}{loss}                                             % loss
\DeclareMathOperator{\signrank}{signrk}                                     % signrank

\let\emptyset\varnothing

\newcommand{\bcube}{{\{ \pm 1 \}^{\cX}}}

\newcommand{\amap}{\xrightarrow{\bZ_2}}
\DeclareMathOperator{\Ind}{Ind}
\newcommand{\IndZ}{\Ind_{\bZ_2}}
\newcommand{\coIndZ}{\coInd_{\bZ_2}}
\DeclareMathOperator{\coInd}{coInd}
\title{Sign-Rank, Index, and List Replicability:  Connections and Separations}

\author{
    Ari Blondal\thanks{McGill University, \texttt{ari.blondal@mail.mcgill.ca, hamed.hatami@mcgill.ca}. Hamed Hatami is supported by an NSERC grant.} \and Hamed Hatami\footnotemark[1] \and
    Pooya Hatami~\thanks{Ohio State University, \texttt{\{hatami.2, lalov.1, tretiak.2\}@osu.edu}} \and  Chavdar Lalov\footnotemark[2] \and Sivan Tretiak\footnotemark[2]
}

\date{July 2026}

\begin{document}

\maketitle

\begin{abstract}
    In learning theory, the sign-rank of a binary concept class captures the smallest dimension in which it can be represented by points and halfspaces. Despite tremendous interest, lower bounds on sign-rank are notoriously difficult to come by. Two recent approaches to the problem establish lower bounds on sign-rank by measures that are easier to analyze: the $\bZ_2$-index and the list replicability number.

    We order these measures, showing that the $\bZ_2$-index is upper-bounded by a linear function of the list replicability number. As a main consequence, we obtain a strong separation between sign-rank and $\bZ_2$-index, thereby resolving a question of Frick, Hosseini, and Vasileuski.

    This motivates a thorough study of list replicability, the stronger of the two lower-bounding measures. We establish upper bounds on the list replicability number by two combinatorial measures: height and minimum star number. We also prove a fundamental composition result, showing that the product of two concept classes has list replicability number bounded by the sum of the list replicability numbers of the two classes.
\end{abstract}

\newpage 

\section{Introduction}

In this paper, we study sign matrices and finite concept classes via three complementary notions of complexity: a geometric notion (sign-rank), a topological notion ($\Z_2$-index), and an algorithmic learning-theoretic notion (list replicability). We show that the list replicability number lies between $\Z_2$-index and sign-rank, and use this to obtain a strong separation between $\Z_2$-index and sign-rank.  We also show that several existing $\Z_2$-index upper bounds in fact hold at the stronger level for list replicability, and develop new upper bounds, lower bounds, and closure properties for list replicability itself.

\paragraph{Sign-rank.} Sign-rank is a fundamental and well-studied parameter in learning theory that captures the smallest dimension in which a binary classification problem admits a linear representation. Formally, the \emph{sign-rank} of a matrix $A$ with entries in $\set{+1,-1}$ is the minimum rank of a real matrix $B$ with  $\sign(B_{i,j}) = A_{i,j}$ for all entries $i,j$.  

Given a binary concept class $\cC \subseteq \{\pm 1\}^{\cX}$ over a \emph{finite} domain~$\cX$, the \emph{sign-rank} $\signrank(\cC)$ is the sign-rank of the $|\cC| \times |\cX|$ matrix $A$ defined by $A_{c, x} = c(x)$. Equivalently, it is the smallest~$d$ for which there exist embeddings $\{u_c \in \mathbb{R}^d\}_{c \in \cC}$ and $\{v_x \in \mathbb{R}^d\}_{x \in \cX}$ satisfying
\begin{equation}
\label{eq:dimension_complexity}
c(x) = \sign \inp{u_c}{v_x}
\quad \text{for all } c \in \cC,\; x \in \cX.
\end{equation}

Classical results on sign patterns of polynomials imply that most $N \times N$ sign matrices have sign-rank $\Omega(N)$~\cite[Lemma~22]{alon2016sign}. Nevertheless, proving even super-constant lower bounds for well-structured matrices that lack pseudorandom properties has remained elusive~\cite{hatami2022lower}. A promising new approach was recently proposed by Frick, Hosseini, and Vasileuski~\cite{frick2026signrank}, who developed a topological framework for sign-rank lower bounds based on the $\ZZ_2$-index of a space associated with the sign matrix. We now describe this framework.

\paragraph{The $\ZZ_2$-index of a concept class.} A distribution $\mu$ over $\cX \times \{\pm 1\}$ is \emph{realizable} by a concept $c \in \cC$ if every pair $(x,b)$ in its support satisfies $b = c(x)$. Given $\cC$, let $\cC^{\pm} \coloneqq \cC \cup \{-c : c \in \cC\}$ denote its \emph{antipodal completion}, and let $\Delta_{\cC^{\pm}}$ be the set of all distributions realizable by some $c \in \cC^{\pm}$. Equipped with the total variation metric, $\Delta_{\cC^{\pm}}$ becomes a topological space.

Recall that a \emph{$\ZZ_2$-action} on a topological space $X$ is a continuous involution $\tau\colon X \to X$ (i.e., $\tau \circ \tau = \mathrm{id}$); it is \emph{free} if $\tau$ has no fixed points. The two examples relevant to us are:
\begin{itemize}
\item The space $\Delta_{\cC^{\pm}}$ with the label-negation action $\mu \mapsto -\mu$, where $(-\mu)(x,b) \coloneqq \mu(x,-b)$.  
\item The unit sphere $\mathbb{S}^d \subset \RR^{d+1}$ with the antipodal action $x \mapsto -x$.
\end{itemize}

Let $X$ be a topological space with a free $\ZZ_2$-action $\tau$. A \emph{continuous} map $\Phi: X \to\mathbb S^d$ is \emph{$\ZZ_2$-equivariant} if it preserves the $\ZZ_2$-action: $\Phi(\tau(x)) = -\Phi(x)$ for all $x \in X$. The \emph{$\ZZ_2$-index} of $X$ is
\[
\Ind_{\ZZ_2}(X) \;\coloneqq\; \min\bigl\{d : \text{there exists a $\ZZ_2$-equivariant map } X \to \mathbb{S}^d\bigr\}.
\]
In particular, $\Ind_{\ZZ_2}(\mathbb{S}^n) = n$; this is in fact equivalent to the Borsuk--Ulam theorem.

Frick, Hosseini, and Vasileuski~\cite{frick2026signrank} defined the \emph{$\ZZ_2$-index of a concept class} $\cC$ as $\Ind_{\ZZ_2}(\cC) \coloneqq \Ind_{\ZZ_2}(\Delta_{\cC^{\pm}})$, where $\Delta_{\cC^{\pm}}$ is equipped with the label-negation $\mathbb{Z}_2$-action described above. 

Dual to the $\ZZ_2$-index is the \emph{$\ZZ_2$-coindex} of a concept class $\cC$, which is the dual invariant obtained by reversing the direction of the equivariant map. Namely, $\coIndZ(\cC)$ is the largest $d$ for which there is a $\ZZ_2$-equivariant map $\mathbb S^d\to \Delta_{\cC^\pm}$. The $\Z_2$-coindex was previously studied under the name \emph{spherical dimension} in~\cite{chornomaz2025spherical}. Note that, by the Borsuk--Ulam theorem, \[\coInd_{\ZZ_2}(\cC)\le \Ind_{\ZZ_2}(\cC).\]

\paragraph{Connection to sign-rank.}
As observed in ~\cite{frick2026signrank}, the sign-rank decomposition~\eqref{eq:dimension_complexity} naturally produces a $\ZZ_2$-equivariant map from $\Delta_{\cC^{\pm}}$ to $\mathbb{S}^{d-1}$: Suppose $\signrank(\cC) \le d$, with embeddings $\{u_c\}_{c \in \cC}$ and $\{v_x\}_{x \in \cX}$ as in~\eqref{eq:dimension_complexity}. The antipodal completion does not increase sign-rank, since we may set $u_{-c} \coloneqq -u_c$. The domain embedding $x \mapsto v_x$ extends by linearity to a map $\psi\colon \Delta_{\cC^{\pm}} \to \RR^d$, defined as $\psi(\mu) \;\coloneqq\; \Ex_{(x,b) \sim \mu}  [ b\, v_x]$.  If $\mu$ is realizable by $c \in \cC^{\pm}$, then every $(x,b) \sim \mu$ satisfies $\inp{u_c}{b v_x}=|\inp{u_c}{v_x}|>0$ with probability $1$, and therefore,  $\psi(\mu) \neq 0$. Hence, we can normalize $\Phi(\mu) \coloneqq \psi(\mu)/\norm{\psi(\mu)}$ to obtain a continuous map $\Phi\colon \Delta_{\cC^{\pm}} \to \mathbb{S}^{d-1}$. Since $\psi$ is linear and label-negation reverses the signed weights, $\psi(-\mu) = -\psi(\mu)$, and hence $\Phi(-\mu) = -\Phi(\mu)$, i.e., $\Phi$ is $\ZZ_2$-equivariant. Therefore,
\begin{equation}
\label{eq:index_le_signrank}
\Ind_{\ZZ_2}(\cC) \;\le\; \signrank(\cC) - 1.
\end{equation}

In light of the construction above, the $\ZZ_2$-index can be viewed as a topological relaxation of sign-rank: while sign-rank requires the map $\psi$ from realizable distributions to $\mathbb{R}^d \setminus \{0\}$ to be \emph{linear}, the $\ZZ_2$-index asks only for a \emph{continuous} $\ZZ_2$-equivariant map with no linearity requirement.

\paragraph{List replicability.}
There are various formalizations of replicability in learning theory, most of which build off of the well-established notion of probably approximately correct (PAC) learning. An $(\epsilon,\delta)$-\emph{PAC learning} algorithm for a concept class $\cC$ receives $n \coloneqq n_\cC(\epsilon,\delta)$ i.i.d.\ examples from an unknown realizable distribution $\mu$ and, with probability at least $1-\delta$, produces a hypothesis $h\colon \cX \to \{\pm 1\}$ with \emph{population loss}
\[
\loss_\mu(h) \;\coloneqq\; \Pr_{(x,b) \sim \mu}[h(x) \neq b] \;\le\; \epsilon.
\]
Given $L \in \mathbb{N}$, such an algorithm is \emph{$L$-list-replicable} if for every realizable distribution $\mu \in \Delta_\cC$, the output hypothesis belongs to a small list $\cL_\mu = \{h_1,\ldots,h_L\}$ with probability at least $1-\delta$. The list may depend on $\mu$, but its size must not. The \emph{list replicability number} $\LR(\cC)$, introduced by~\cite{chase2023replicabilitystabilitylearning,dixon2023listandcertificate}, is the smallest $L$ for which an $L$-list-replicable $(\epsilon,\delta)$-PAC learner exists for all $\epsilon, \delta > 0$.

\subsection{Our results}

\subsubsection{List replicability bridges $\Z_2$-index and sign-rank.}
Our main result connects the $\ZZ_2$-index to list replicability.

\begin{theorem}[Index is controlled by list replicability]
\label{thm:index_vs_LR}
For every concept class $\cC \subseteq \{\pm 1\}^\cX$ over a finite domain $\cX$,
\[
\Ind_{\ZZ_2}(\cC) \;\le\; 2\,\LR(\cC) - 1.
\]
\end{theorem}

It is shown in~\cite{blondal2026borsuk} that, over finite domains,
$\LR(\cC) \le \signrank(\cC)$.  Combining this with
\Cref{thm:index_vs_LR} and the Borsuk--Ulam inequality
$\coInd_{\ZZ_2}(\cC)\le \Ind_{\ZZ_2}(\cC)$ gives the following chain of inequalities:
\begin{equation}
\label{eq:chain}
\frac{1}{2}\coInd_{\ZZ_2}(\cC) \; \le\; \frac{1}{2}\Ind_{\ZZ_2}(\cC) \;<\; \LR(\cC) \;\le\;   \signrank(\cC). 
\end{equation}

Frick, Hosseini, and Vasileuski~\cite[Question~7]{frick2026signrank} asked whether there exists
a function $f$ such that
\[
    \signrank(\cC) \le f(\Ind_{\ZZ_2}(\cC))
\]
for every finite concept class $\cC$. A similar question, whether $\signrank(\cC)\leq f(\LR(\cC))$ for some function $f$, was raised in \cite{blondal2026borsuk}.  We use \Cref{thm:index_vs_LR} in conjunction with existing results about sign-rank to resolve both problems.

\begin{theorem}[Sign-rank is not bounded by index]
\label{thm:main}
There exists a family of $N \times N$ sign matrices with list replicability number at most $3$ and $\ZZ_2$-index at most $5$, while their sign-rank grows polynomially in $N$.
\end{theorem}

The matrices are incidence matrices of finite projective planes.  Their sign-rank is known to be polynomially large by the lower bound of
Alon, Moran, and Yehudayoff~\citep{alon2016sign}. We show in \cref{thm:lr_of_PZq2} that their list replicability number is
at most $3$.  \cref{thm:index_vs_LR} then gives
$\Ind_{\ZZ_2}\le 5$. 

\subsubsection{Extremal classes.}

Recall that a set $S \subseteq \cX$ is \emph{shattered} by $\cC \subseteq \{\pm 1\}^{\cX}$ if $\cC|_S = \{\pm 1\}^S$, and the \emph{VC dimension} of $\cC$ is the size of the largest set that it shatters. An immediate topological ramification of shattering a set $S$ of size $d$ is that the set of distributions in $\Delta_{\cC^{\pm}}$ supported on $S$ is homeomorphic to $\mathbb{S}^{d-1}$ (via the identification of distributions with $\ell_1$-unit vectors in $\RR^\cX$), and therefore $\coInd_{\ZZ_2}(\cC) \geq d - 1$. Consequently, every concept class satisfies 
\[\VCdim(\cC) - 1 \le \coInd_{\ZZ_2}(\cC) \le \Ind_{\ZZ_2}(\cC). \]

A set $S \subseteq \cX$ is \emph{strongly shattered} by $\cE \subseteq \{\pm 1\}^{\cX}$ if there is a fixed labelling $a \in \{\pm 1\}^{\cX \setminus S}$ with
\[
\{c|_S : c \in \cE,\; c|_{\cX \setminus S} = a\} = \{\pm 1\}^S.
\]
A class $\cE$ is \emph{extremal} if every shattered set is also strongly shattered. Many natural concept classes are extremal or admit natural extremal extensions; see~\cite[Section~3.2]{chalopin2022unlabeled} for a list of examples.

Extremal classes are among the few cases where list replicability is well understood. The authors in~\cite{blondal2026simplicial} showed that for an extremal class $\cE \subseteq \{\pm 1\}^{\cX}$ over a finite domain,
\begin{equation}
\label{eq:extremal}
\LR(\cE) =
\begin{cases}
\VCdim(\cE) & \text{if } \cE = \{\pm 1\}^{\cX},\\
\VCdim(\cE) + 1 & \text{otherwise.}
\end{cases}
\end{equation}
Combined with \Cref{thm:index_vs_LR} and the general bound $\coInd_{\ZZ_2}(\cC) \ge \VCdim(\cC) - 1$, this pins down all three parameters up to a factor of two: for every extremal class $\cE$,
\begin{equation}
\label{eq:extremal_chain}
\VCdim(\cE) - 1 \;\le\; \coInd_{\ZZ_2}(\cE) \;\le\; \Ind_{\ZZ_2}(\cE) \;\le\; 2\LR(\cE) - 1 \leq 2\VCdim(\cE) + 1.
\end{equation}

\subsubsection{Height and eluder dimension.}\label{subsec:height_eluder}

Invariants which capture a notion of height for topological spaces (such as the Stiefel--Whitney height) are standard tools for bounding the $\Z_2$-index/coindex. In the same vein, Frick, Hosseini, and Vasileuski~\cite{frick2026signrank} introduced a combinatorial notion of height as a way to upper-bound the $\Z_2$-index of a sign matrix. Using this approach, they showed that random $N \times N$ sign matrices and the Hadamard matrix have $\Z_2$-index $O(\log N)$.

We instead use a version of their height definition adapted to partial concept classes $\cC\subseteq\set{\pm1,\star}^{\cX}$ (both versions are equivalent up to a constant factor). For two partial concepts $h_1, h_2 \in \{\pm 1, \star\}^{\cX}$, define their \emph{junction} by
\[
    (h_1 \cap h_2)(x) =
    \begin{cases}
    h_1(x) & \text{if } h_1(x) = h_2(x),\\
    \star & \text{otherwise}.
    \end{cases}
\]
We write $g \preceq h$ if $g = g \cap h.$
Write \(g\prec h\) if \(g\preceq h\) and \(g\neq h\). Define the \emph{junction closure}  of a partial concept class $\cC \subseteq \set{\pm1,\star}^\cX$ as:
\[
    \cJ(\cC)
    =
    \left\{
        \bigcap_{c \in S} c : \emptyset \neq S \subseteq \cC
    \right\}.
\]
The \emph{height} of $\cC$, denoted $\operatorname{H}(\cC)$, is the length $m$ of the longest strict chain $h_1\prec h_2\prec \cdots \prec h_m$ in $\cJ(\cC)$.\footnote{Given $\cC \subseteq \{\pm 1\}^{\cX}$, let $\cS_\cC \coloneqq \{c^+, c^- : c \in \cC\}$, where $c^+ \coloneqq \{x : c(x) = 1\}$ and $c^- \coloneqq \{x : c(x) = -1\}$, and let $\cI_\cC$ be the set of all finite non-empty intersections of elements of $\cS_\cC$. Frick, Hosseini, and Vasileuski~\cite{frick2026signrank} define the height $h_{\mathrm{FHV}}(\cC)$ as the length of the longest strict chain in $\cI_\cC$. Up to the convention of whether the empty set is included in $\cI_\cC$,
\[
    h_{\mathrm{FHV}}(\cC) \;\le\; \height(\cC^{\pm}) \;\le\; 2\,h_{\mathrm{FHV}}(\cC) - 1.
\]}

In view of the aforementioned bound of Frick, Hosseini, and Vasileuski on the $\ZZ_2$-index by height, and our \Cref{thm:index_vs_LR} bounding the $\ZZ_2$-index by $\LR$, it is natural to ask how the two upper-bounding quantities, height and $\LR$, compare. Our next result clarifies this by showing that height also upper-bounds $\LR$.

\begin{restatable}[Height controls list replicability]{theorem}{heightboundslr}
\label{thm:height_bounds_LR}
For every finite partial concept class $\cC \subseteq \set{\pm1,\star}^\cX$,
\[
    \LR(\cC) \le \height(\cC).
\]
\end{restatable}

The proof of \Cref{thm:height_bounds_LR}, as found in \cref{sec:height_vs_lr}, is algorithmic. We construct a learner whose output hypothesis is likely to lie on a single chain in the junction closure of the class. Since every such chain has length at most $\height(\cC)$, the learner is $\height(\cC)$-list-replicable.

In \cref{sec:height_vs_lr}, we combine \cref{thm:height_bounds_LR} with a result of \cite{frick2026signrank} to conclude that random $N\times N$ sign matrices have a list replicability number bounded by $O(\log N)$. 

The combinatorial notion of height given by $\height(\cC)$ has a useful equivalence to a learning-theoretic parameter known as eluder dimension. The eluder dimension was introduced by Russo and Van Roy~\cite{russo2013eluder} as a sequential notion of independence for function classes. Informally, it is the maximum length of a sequence of points such that, at each step, the label at the next point is not determined by the labels on the preceding points.

Fix a finite partial concept class $\cC\subseteq\{\pm1,\star\}^{\cX}$ and a base concept $c'\in\cC$. Write $\supp(c') := \{x\in\cX : c'(x)\neq \star\}$. The \emph{eluder dimension} of
$\cC$ \emph{relative} to $c'$, denoted $\operatorname{Edim}(\cC,c'),$
is the largest $m$ for which there exist points $x_1,\ldots,x_m\in \supp(c')$ and concepts $
c_1,\ldots,c_m\in\mathcal C
$ such that, for every \(i\in[m]\),
\[
    c_i(x_i)\neq c'(x_i),
    \;\text{\textnormal{($c_i(x_i)=\star$ allowed)}}
\]
while at every earlier position $j<i$,
\[
    c_i(x_j)=c'(x_j).
\]
The \emph{eluder dimension} of $\cC$ is $\edim(C)\coloneq\sup_{c'\in\cC}\operatorname{Edim}(\cC,c')$.

\begin{proposition}[Height and eluder dimension are equivalent]\label{prop:height_eluder}
For any finite partial concept class $\cC \subseteq \set{\pm1,\star}^\cX$,
\[
    \operatorname{H}(\cC) = \operatorname{Edim}(\cC)+1.
\]
\end{proposition}
We prove a slightly stronger form of this in \cref{prop:eluder}. The result allows us to translate the bound $\LR(\cC) \leq \height(\cC)$ into the existing literature surrounding eluder dimension, as will be seen in the following section.

\subsubsection{The star number.}

The \emph{star number} is a classical combinatorial parameter from learning theory that measures the size of the largest local star around a target labelling~\cite{hanneke2015minimax}.

\begin{definition}[Star number]\label{def:star_number}
For a class $\cC \subseteq \{\pm 1\}^{\cX}$ and a center $h \in \{\pm 1\}^{\cX}$, let $s_h(\cC)$ denote the largest $m$ for which there exist distinct points $x_1,\ldots,x_m \in \cX$ and concepts $c_0, c_1, \ldots, c_m \in \cC$ such that
\[
    c_i(x_j) = h(x_j) \qquad\Longleftrightarrow\qquad i \neq j
\]
for every $i \in \{0,\ldots,m\}$ and $j \in [m]$.
\end{definition}

Thus $c_0$ agrees with $h$ on all of $x_1,\ldots,x_m$, while each other $c_i$ disagrees with $h$ exactly at $x_i$. We consider both extremes over the choice of center:
\[
    s_{\min}(\cC) \coloneqq \min_{h \in \{\pm 1\}^{\cX}} s_h(\cC)
    \qquad\text{and}\qquad
    s_{\max}(\cC) \coloneqq \max_{h \in \{\pm 1\}^{\cX}} s_h(\cC).
\]

Hanneke \cite{hanneke2024star} showed that the star number characterizes
intersection-closed structure: the minimum star number
\(s_{\min}(\cC)\) of a concept class \(\cC\) equals the least possible VC
dimension of a \emph{generalized intersection-closed} class \(\cH\)
containing \(\cC\).

\begin{definition}[Generalized intersection-closed] We say that \(\cH\subseteq\set{\pm 1}^{\cX}\) is
\emph{generalized intersection-closed} if there exists a center
\(h_\star\in\set{\pm 1}^{\cX}\) such that, for every nonempty finite
\(\cA\subseteq\cH\), the concept
\[
    \left(\bigwedge_{h_\star}\cA\right)(x)
    \coloneqq
    \begin{cases}
        h_\star(x),
        & \text{if } c(x)=h_\star(x)\text{ for every }c\in\cA,\\
        -h_\star(x),
        & \text{otherwise}
    \end{cases}
\]
also belongs to \(\cH\).
\end{definition}
This notion generalizes the classical definition
of an intersection-closed class, which corresponds to the case where
\(h_\star\) is the all-\(1\) hypothesis.

We show that all finite concept classes can be list-replicably learned with list size $O(s_{\min})$, by combining Hanneke's
characterization with known embeddings into extremal concept classes.

\begin{theorem}\label{thm:lr_smin}
For every finite binary concept class \(\cC\subseteq\set{\pm 1}^{\cX}\),
\[
    \LR(\cC)\leq 11s_{\min}(\cC)+1.
\]
\end{theorem}

\begin{proof}
By Hanneke's characterization of intersection-closed classes, \(\cC\)
embeds into a generalized intersection-closed class \(\cH\) satisfying
\[
    \VCdim(\cH)=s_{\min}(\cC);
\]
see \cite[Theorem~19 and Remark~24]{hanneke2024star}.

Rubinstein and Rubinstein show that every generalized
intersection-closed class of VC dimension \(d\) embeds into an extremal
class of VC dimension at most \(11d\); see
\cite[Theorems~4.1 and~4.4]{rubinstein2022unlabelled}. Applying this
result to \(\cH\), we obtain an extremal class
\(\cE\subseteq\set{\pm 1}^{\cX}\) such that
\(\cC\subseteq\cH\subseteq\cE\) and
\[
    \VCdim(\cE)
    \leq
    11\VCdim(\cH)
    =
    11s_{\min}(\cC).
\]
This combined embedding statement is also recorded explicitly in
\cite[Corollary~27]{hanneke2024star}.

Finally, by \Cref{eq:extremal}, every finite extremal class satisfies 
\[
    \LR(\cE)\leq \VCdim(\cE)+1. 
\]
By monotonicity of the list replicability number under taking subclasses, we conclude
\[
    \LR(\cC)
    \leq
    \LR(\cE)
    \leq
    \VCdim(\cE)+1
    \leq
    11s_{\min}(\cC)+1.\qedhere
\]
\end{proof}

\begin{remark}
It is easy to see that the maximum star number is upper bounded by eluder dimension and height\footnote{\label{note1}In fact, Li, Kamath, Foster, and Srebro~\cite{li2022understanding} showed that the eluder dimension is characterized by the maximum of star number and \emph{threshold dimension}: $\max\{s_{\max}(\cC), \tdim(\cC)\}\leq \edim(\cC)\leq4^{\max\{s_{\max}(\cC), \tdim(\cC)\}}$.  }. In particular, for every concept class $\cC$, we have
\begin{equation}
s_{\min}(\cC)\leq s_{\max}(\cC) \leq \edim(\cC)=\height(\cC)-1. 
\end{equation}
Thus, for total classes, \cref{thm:lr_smin} is a strengthening, ignoring constant factors, of the height/eluder bound \(\LR(\cC)\le \height(\cC)\) of \cref{thm:height_bounds_LR}. The height bound, however, applies in the more general setting of partial concept classes. 
\end{remark}

\paragraph{Star number and sign-rank.}
Li, Kamath, Foster, and Srebro~\cite{li2022understanding} asked whether there exist concept classes with bounded $s_{\max}$ but arbitrarily large sign-rank. They also conjectured the stronger statement that there exist concept classes with constant $\edim$ and arbitrarily large sign-rank. We disprove both of these conjectures: sign-rank is bounded by a function of $s_{\max}$ alone.

\begin{theorem}
\label{thm:starnumber_signrank}
For every sign matrix $A$ with $s_{\max}(A) \le d$,
\[
\signrank(A) \;\le\; 2\,f(d)^2 \;\le\; 2^{2^{O(d\log d)}},
\]
where $f$ is the function from~\cite[Theorem~2.1]{atminas2022classes}.
\end{theorem}
\begin{proof}
Let $\cC\coloneqq \cC_A\in \{\pm 1\}^\cX$ denote the binary concept class corresponding to $A$, where $\cX$ is the set of columns of $A$. Consider the bipartite graph $G_\cC$ with parts $\cC$ and $\cX$, where $c \in \cC$ is adjacent to $x \in \cX$ if and only if $c(x) = -1$. Let $d \coloneqq  s_{\max}(\cC)$.  Then since $s_{\mathbf{1}}(\cC)\le d$ and $s_{\mathbf{-1}}(\cC) \le d$, where $\mathbf{1}$ and $\mathbf{-1}$ denote the all-ones and all-minus-ones hypotheses\footnote{Conversely, since every $h \in \{\pm 1\}^{\cX}$ satisfies $s_h(\cC) \le s_{\mathbf{1}}(\cC) + s_{\mathbf{-1}}(\cC)$, we have $s_{\max}(\cC) \le s_{\mathbf{1}}(\cC) + s_{\mathbf{-1}}(\cC)$.}, neither $G_\cC$ nor its bipartite complement contains an induced matching of size $d+2$. Atminas~\cite[Theorem~2.1]{atminas2022classes} showed that there is a function $f(d) \le 2^{2^{O(d\log d)}}$ such that whenever a bipartite graph $G = (A \cup B, E)$ and its bipartite complement both omit induced matchings of size $d$, the parts admit partitions $A = A_1 \cup \cdots \cup A_u$ and $B = B_1 \cup \cdots \cup B_u$ with $u \le f(d)$ such that every induced subgraph $G[A_i, B_j]$ contains no induced copy of $2K_2$.

We combine this with the following standard fact about $2K_2$-free bipartite graphs.

\begin{proposition}[{\cite{MahadevPeled1995}; see \cite[Theorem 6.3]{HeggernesKratsch2007Certifying}}]
\label{prop:2K2_thr}
If a bipartite graph $G = (A \cup B, E)$ contains no induced matching with two edges, then there is an ordering $B = \{b_1,\ldots,b_m\}$ such that $a b_i \in E$ implies $a b_j \in E$ for all $j\ge i$.
\end{proposition}

The adjacency pattern of \Cref{prop:2K2_thr} is realized by points and thresholds on a line, so each block $G[A_i, B_j]$ in Atminas's partition, viewed as a sign matrix, has sign-rank at most $2$. The matrix $A$ is partitioned into a $u \times u$ grid of such blocks with $u \le f(d)$, and since sign-rank is subadditive under both horizontal and vertical concatenation, we conclude $\signrank(A) \le 2u^2 \le 2f(d)^2$.
\end{proof}

\subsubsection{Separation of coindex and list replicability for partial classes.}
All three parameters, $\Ind_{\ZZ_2}$, $\coInd_{\ZZ_2}$, and $\LR$, extend naturally to partial concept classes $\cC \subseteq \{\pm 1,\star\}^{\cX}$. A distribution $\mu$ over $\cX \times \{\pm 1\}$ is realizable by such a $c$ if it is supported on pairs $(x,c(x))$ with $c(x) \neq \star$.   

One advantage of allowing partial concept classes is that they are flexible enough to encode known classical topological examples.  In particular, building on an observation of Frick, Hosseini, and Vasileuski~\cite{frick2026signrank}, standard free $\ZZ_2$-spaces with bounded $\Z_2$-coindex but unbounded $\Z_2$-index can be realized by partial concept classes. Combining this with \Cref{thm:index_vs_LR}, we obtain a dimension-free separation between $\coInd_{\Z_2}$ and $\LR$. 

\begin{restatable}{corollary}{LRvscoInd}\label{cor:separation_LR_coInd}
There exist partial concept classes $\cC \subseteq \{\pm 1,\star\}^n$ with
\[
\coInd_{\ZZ_2}(\cC) = O(1) \qquad\text{and}\qquad \LR(\cC) = \omega(1).
\]
\end{restatable}
 
We outline the argument below, leaving the details for \cref{sec:cor_and_connections}. The argument builds on an observation of Frick, Hosseini, and Vasileuski that every finite free $\Z_2$-simplicial complex arises as $\Delta_{\cC^\pm}$ for some partial concept class $\cC$~\cite{frick2026signrank}. Moreover, by Illman's equivariant triangulation theorem~\cite{illman_smooth_1978}, every compact smooth free $\ZZ_2$-manifold admits a finite $\ZZ_2$-equivariant triangulation. This gives a recipe for importing standard examples from equivariant topology to partial concept classes.
Applying this reduction to the spaces $\mathbb{RP}^{2N-1}$, 
with the free $\ZZ_2$-action induced by multiplication by $i$ on $\mathbb{C}^N$, one obtains partial concept classes with bounded $\coInd_{\ZZ_2}$ and unbounded $\Ind_{\ZZ_2}$.  Combined with \Cref{thm:index_vs_LR}, this gives a dimension-free separation between $\coInd_{\ZZ_2}$ and $\LR$.

The authors of \cite{frick2026signrank} have privately communicated to us that they have independently proved a separation of $\ZZ_2$-coindex from the $\ZZ_2$-index for partial concept classes using an explicit triangulation of $\mathbb{RP}^{2N-1}$, avoiding Illman's theorem.  

\subsubsection{List replicability under joins and concatenations.}\label{subsec:joins_concats}
We next record quantitative composition properties of list replicability. 

Given $\cC_1 \subseteq \{\pm 1\}^{X_1}$ and $\cC_2 \subseteq \{\pm 1\}^{X_2}$ over disjoint finite domains $X_1$ and $X_2$, define their \emph{join} $\cC_1 * \cC_2 \subseteq \{\pm 1\}^{X_1 \sqcup X_2}$ as
\[
\cC_1 * \cC_2 \;\coloneqq\; \bigl\{c \in \{\pm 1\}^{X_1 \sqcup X_2} : c|_{X_1} \in \cC_1 \text{ and } c|_{X_2} \in \cC_2\bigr\}.
\]
The terminology is motivated by the fact that $\Delta_{\cC_1^{\pm} * \cC_2^{\pm}}$ is homeomorphic to the topological join $\Delta_{\cC_1^{\pm}} * \Delta_{\cC_2^{\pm}}$: since $X_1$ and $X_2$ are disjoint, every realizable distribution for $\cC_1 * \cC_2$ is a convex combination of a distribution supported on $X_1$ and one supported on $X_2$ \cite[Lemma 23]{chornomaz2025spherical}. Note that the $\ZZ_2$-index is subadditive, and the coindex is superadditive under joins (see, e.g.,~\cite{matousek2003borsuk}),
\begin{equation}
\label{eq:join_index}
\Ind_{\ZZ_2}(\cC_1 * \cC_2) \;\le\; \Ind_{\ZZ_2}(\cC_1) + \Ind_{\ZZ_2}(\cC_2) + 1,
\end{equation}
and 
\begin{equation}
\label{eq:join_coindex}
\coInd_{\ZZ_2}(\cC_1 * \cC_2) \;\ge\; \coInd_{\ZZ_2}(\cC_1) + \coInd_{\ZZ_2}(\cC_2) + 1. 
\end{equation}

In light of \Cref{thm:index_vs_LR}, it is natural to ask how list replicability behaves under joins. It is not difficult to prove $\LR(\cC_1 * \cC_2) \le (\LR(\cC_1)+1)\cdot (\LR(\cC_2)+1)$ by essentially running the list-replicable algorithms for $\cC_1$ and $\cC_2$ separately on the corresponding parts of the sample and combining their outputs into a single hypothesis on $X_1 \sqcup X_2$: the resulting list consists of all pairs of hypotheses from the two lists. We improve this multiplicative bound to an additive one, which is sharp.

\begin{theorem}
\label{thm:LR_join}
Given two concept classes $\cC_1$ and $\cC_2$ over disjoint finite domains, we have
\[\LR(\cC_1 * \cC_2) \le \LR(\cC_1) + \LR(\cC_2). \] 
\end{theorem}

The proof uses a quantile-style coupling to synchronize the choices of the two LR learners. Instead of taking all pairs of possible outputs, it aligns the two output distributions on a common interval, so only linearly many pairs can appear. The full proof can be found in \cref{sec:properties_LR}.

For matrices, \cref{thm:LR_join} gives useful decomposition rules. Given two concept classes $\cC_1$ and $\cC_2$ over the same domain $X$, it follows from the definition that $\LR(\cC_1 \cup \cC_2) \le \LR(\cC_1) + \LR(\cC_2)$, since one can run list-replicable learners for both classes and take the union of the resulting lists (see also \cref{lemma:LR_union}). This operation corresponds to the vertical concatenation of the associated sign matrices. On the other hand, horizontal concatenation is controlled by joins: if $A_1$ and $A_2$ have the same number of rows, then the concept class associated with $[A_1 \ A_2]$ is naturally a subclass of the join of the two corresponding concept classes. We thus obtain subadditivity of list replicability under both horizontal and vertical concatenation of matrices.

\begin{corollary}
\label{cor:LR_blocks}
For sign matrices with matching dimensions,
\[
\LR\bigl(\begin{bmatrix} A_1 & A_2 \end{bmatrix}\bigr) \;\le\; \LR(A_1) + \LR(A_2)
\qquad\text{and}\qquad
\LR\bigl(\begin{bmatrix} A_3 \\ A_4 \end{bmatrix}\bigr) \;\le\; \LR(A_3) + \LR(A_4).
\]
\end{corollary}

\subsection{Technical overview of \cref{thm:index_vs_LR} and \cref{thm:main}}\label{sec:technical_overview}
In this section, we briefly describe how we prove \cref{thm:index_vs_LR} and \cref{thm:main}. Complete proofs are in \cref{sec:separate_sgn_ind}.

\paragraph{Bounding index by LR.} \cref{thm:index_vs_LR} states that for any concept class $\cC$,
\begin{align*}
\operatorname{\Ind}\limits_{\bZ_2}(\cC) \leq 2 \LR(\cC) - 1.
\end{align*}

To prove this, we first use the fact that a list-replicable algorithm with list size $L$ gives an antipodal-free open cover of the distribution space $\Delta_{\cC^\pm}$, where each point is contained in at most $2L$ open sets (see \cref{thm:closed-cover-lr}). This bounded overlap allows us to use a partition of unity to construct a continuous map from $\Delta_{\cC^{\pm}}$ into $\mathbb{R}^{2L} \setminus \{0\}$.

The antipodal symmetry of the cover ensures that this map is $\bZ_2$-equivariant, while antipodal-freeness ensures that the image avoids the origin. After normalizing, we obtain a $\bZ_2$-equivariant map
\[
    \Delta_{\cC^{\pm}} \amap \mathbb{S}^{2L-1}.
\]
Therefore, by the definition of the $\bZ_2$-index,
\[
    \operatorname{Ind}_{\mathbb{Z}_2}(\Delta_{\cC^{\pm}}) \leq 2L - 1.
\]

\paragraph{Separating index from sign-rank.}
The separation of sign-rank and index in \cref{thm:main} hinges on a particular family of matrices $B_q$ for which $\signrank(B_q)$ grows polynomially in $q$, while $\Ind_{\bZ_2}(B_q)$ is bounded. This family was first used by Alon, Moran, and Yehudayoff to separate sign-rank and VC dimension \citep{alon2016sign}.

\begin{definition}[Finite Projective Plane]\label{def:PZq2}
    Let $\PG(2,q)$ be the finite projective plane of order $q$.
    It is an incidence geometry $(\cP, \cL)$ with $N = q^2 + q + 1$ points and lines such that:
    \begin{enumerate}
        \item Any two distinct lines $\ell_1, \ell_2 \in \cL$ intersect in exactly one point $p \in \cP$.
        \item Any two distinct points $p_1, p_2 \in \cP$ are contained in exactly one line $\ell \in \cL$.
    \end{enumerate}
    Let $B_q$ be the $N \times N$ indicator matrix of this incidence geometry.
    That is, $B_q$ has rows indexed by lines and columns indexed by points, and
    \begin{align*}
        (B_q)_{ij} = \begin{cases}
            1 &\text{if } p_j \in \ell_i\\
            -1 &\text{if } p_j \notin \ell_i.
        \end{cases}
    \end{align*}
\end{definition}

Matrices of this form were useful for the separation of VC dimension and sign-rank because of the polynomial growth of $\signrank(B_q)$, a result which we will apply directly for our separation of index and sign-rank.

\begin{theorem}[\cite{alon2016sign}]\label{thm:poly_growth_sgn_rnk}
    The $N \times N$ indicator matrix $B_q$ of $\PG(2,q)$ satisfies the inequality
    \begin{align*}
        \signrank(B_q) \geq \frac{q^2 - 1}{\sqrt{q}(q-1)} \geq N^{\frac{1}{4}}.
    \end{align*}
\end{theorem}

The remaining half of our separation requires bounding $\Ind_{\bZ_2}(B_q)$ by a constant. This is done by combining \cref{thm:index_vs_LR} with a list-replicable algorithm for $B_q$, which we give in \cref{thm:lr_of_PZq2}.

\paragraph{The algorithm.} Here we treat the matrix $B_q$ as a concept class with concepts labeled by lines and domain points labeled by points. We design a simple list-replicable algorithm for $B_q$ that shows
\begin{align*}
    \LR(B_q) \leq 3.
\end{align*}

The algorithm proceeds as follows. If the sample contains two distinct points labeled $1$, then there is a unique line passing through both of them, so the algorithm outputs that line. Otherwise, if the sample contains only a single point labeled $1$ and that point has low sampling probability, the algorithm ignores it, since doing so affects the error only negligibly with high probability. Finally, if the sample contains exactly one point labeled $1$ and that point has high sampling probability, the algorithm outputs the indicator function of that point. The complete proof of this step can be found in \cref{thm:lr_of_PZq2}.

Combining the two steps, we have that the family $B_q$ has $\Ind_{\bZ_2}(B_q) \leq 5$, but $\signrank(B_q) = \Omega(\sqrt{q})=\Omega (N^{1/4})$.

\subsection{Related work}

\paragraph{Sign-rank.}
The sign-rank of a matrix was introduced by Paturi and Simon~\cite{paturi1986probabilistic}, who observed that $\VCdim(A) \le \signrank(A)$ as a consequence of the VC dimension of halfspaces. Sign-rank has since become a fundamental quantity in theoretical computer science, with connections to learning theory, communication complexity, circuit complexity, combinatorics, discrete geometry, and Banach space theory; see~\cite{hatami2022lower}. Shortly after its introduction, Alon, Frankl, and R\"odl~\cite{alon1985geometrical} used bounds on the number of connected components of real algebraic varieties~\cite{MR161339,MR0200942,MR226281} to prove linear lower bounds on the sign-rank of random matrices. 

For explicit matrices, the VC dimension bound remained the state of the art for nearly two decades, until the breakthrough of Forster~\cite{MR1964645}, who proved the $n \times n$ Hadamard matrix $H_n$ satisfies $\signrank(H_n) \ge \sqrt{n}$, the first super-logarithmic lower bound on the sign-rank of an explicit matrix.   

\paragraph{Topological methods.}
The use of topological obstructions in combinatorics has a rich history, with Lov\'asz's proof of Kneser's conjecture~\cite{lovasz1978kneser} as a landmark example; see~\cite{matousek2003borsuk} for a comprehensive treatment. The common theme is to associate a $\ZZ_2$-space to a combinatorial object and extract consequences from equivariant invariants such as the $\ZZ_2$-index and coindex~\cite{matousek2002topological,csorba2004box,simonyi2006local,simonyi2009local}. The sign-rank lower bound of Frick, Hosseini, and Vasileuski~\cite{frick2026signrank} via the $\ZZ_2$-index fits squarely within this framework.

\paragraph{List replicability.}
Replicability, the requirement that an algorithm produce consistent outcomes when repeated under similar conditions, has become a vibrant research area in learning theory, with various rigorous formulations introduced and studied~\cite{BLM20,malliaris2022unstable,chase2023replicabilitystabilitylearning,bun2023stability,karbasi2023replicability,esfandiari2023replicable,Esfandiarietal23,moran2023bayesian,eaton2024replicable,kalavasis2024replicable,kalavasis2023statistical}. 

A key notion in this area is \emph{global stability}, which emerged from the study of differentially private and online learning~\cite{BLM20,Alon_22_private_and_online}. Chase, Moran, and Yehudayoff~\cite{chase2023replicabilitystabilitylearning} later reformulated global stability in the equivalent language of \emph{list replicability}. Subsequent work has revealed that this notion is intrinsically linked to the geometry and topology of the space of realizable distributions~\cite{chase2023replicabilitystabilitylearning,localborsukulam,BGHH2025stabilitylistreplicabilityagnosticlearners,chornomaz2025spherical,blondal2026borsuk,blondal2026simplicial,blondal2026tightlistreplicabilitybounds}.

\paragraph{Extremal and intersection-closed concept classes.}
The study of extremal and intersection-closed classes is motivated by the fact that their combinatorial structure gives rise to natural learning algorithms, sharp PAC sample-complexity bounds, and simple sample-compression constructions \cite{moran2016labeled,chalopin2022unlabeled,chase2024dual,blondal2026simplicial,helmbold1990learning,haussler1994predicting,floyd1995sample,bendavid1998self,kuhlmann1999teaching,dalmau2003learnability,auer2007new,darnstadt2015optimal,hanneke2016refined,blum2021robust,rubinstein2022unlabelled,hanneke2024star}. 

\paragraph{Relationships between parameters.}   Frick, Hosseini, and Vasileuski~\cite{frick2026signrank} introduced the $\ZZ_2$-index and the closely related parameter $\coInd_{\ZZ_2}(\cC)$, and showed
\[
\VCdim(\cC) - 1 \;\le\; \coInd_{\ZZ_2}(\cC) \;\le\; \Ind_{\ZZ_2}(\cC) \;\le\; \signrank(\cC) - 1.
\] 
Chase, Moran, and Yehudayoff~\cite[Theorem~3]{chase2023replicabilitystabilitylearning} proved that $\VCdim(\cC) \le \LR(\cC)$ for every concept class $\cC$.
The inequality $\LR(\cC) \le \signrank(\cC)$ was conjectured in~\cite{chase2023replicabilitystabilitylearning}, verified there for $\signrank(\cC) = 2$, and later resolved in full by Blondal, Hatami, Hatami, Lalov, and Tretiak~\cite{blondal2026borsuk}.

Finally, the \emph{Littlestone dimension} is a refinement of the VC dimension that characterizes the optimal mistake bound in online learning; in particular, the VC dimension is a lower bound on the Littlestone dimension. A celebrated result of Bun, Livni, and Moran~\cite{BLM20,Alon_22_private_and_online} shows that every class with Littlestone dimension $d$ satisfies
\[
\LR(\cC) \;\le\; 2^{2^{O(d)}}.
\]

\subsection{Concluding remarks and open problems}

\paragraph{How large can the $\ZZ_2$-index be?}
The lower bound $\VCdim(\cC) - 1 \le \Ind_{\ZZ_2}(\cC)$ provides $N \times N$ sign matrices with $\Ind_{\ZZ_2}(A) \ge \log N$, namely those with $\VCdim(A) = \log N$.

Since a typical $N \times N$ sign matrix has sign-rank $\Omega(N)$,  \eqref{eq:chain} leaves a wide gap between this logarithmic lower bound and the polynomial behaviour of sign-rank. It is natural to ask whether the $\ZZ_2$-index or the list replicability number of an $N \times N$ sign matrix can be super-logarithmic, or even polynomial, in $N$. Evidence so far has pointed in the negative direction: Frick, Hosseini, and Vasileuski~\cite{frick2026signrank} showed that a random $N \times N$ sign matrix satisfies $\Ind_{\ZZ_2}(A) = O(\log N)$ with high probability, and that the $N \times N$ Hadamard matrix, the standard example of an explicit matrix with large sign-rank, also satisfies $\Ind_{\ZZ_2}(H_N) = O(\log N)$. Nevertheless, the logarithmic barrier can be surpassed.

\begin{proposition}
\label{prop:log2}
There exist $N \times N$ sign matrices $A$ with
\[
2 \LR(A) - 1 \ge \Ind_{\ZZ_2}(A) \ge \coIndZ(A) \;\ge\; \Omega\!\left(\frac{\log^2 N}{\log\log N}\right).
\]
\end{proposition}
\begin{proof}[Proof of \cref{prop:log2}, originally from \cite{chornomaz2025spherical}]
    Consider the $m \times 2^m$ matrix $\mathcal{U}_m$, containing a column for every sign pattern on the $m$ rows.
    Chornomaz, Moran, and Waknine show that it has $\Z_2$-coindex of at least $m-2$.
    Since the coindex is superadditive with respect to joins (recall \cref{subsec:joins_concats}), we take the join of $k = m / \log_2 m$ copies of $\mathcal{U}_m$.
    \[
        \mathcal{U}_m^k = \mathcal{U}_m * \dots * \mathcal{U}_m,
    \]
    and obtain $\coIndZ(\mathcal{U}_m^k) \geq (m-2)(m) / \log m$. Note $\mathcal{U}_m^k$ has $m^{m / \log m} = 2^m$ rows and $(m / \log  m) 2^m$ columns. Therefore, it is contained in an $N \times N$ sign matrix $A$, where $N = (m / \log m) 2^m$, with
    \begin{align*}
        2 \LR(A) - 1 &\geq \IndZ(A) \geq \coIndZ(A) \geq
        \Omega\left(\frac{m^2}{\log m}\right)
        = \Omega \left(\frac{\log^2 N}{\log\log N}\right).
        \qedhere
    \end{align*} 
\end{proof}

The lower bound provided in \Cref{prop:log2} is currently the strongest known lower bound on the list replicability and the $\mathbb{Z}_2$-index of any $N \times N$ matrix. We conjecture that this bound is essentially tight (see also \cite[Question 37]{frick2026signrank}).

\begin{conjecture}
Every  $N \times M$ sign matrix $A$ satisfies 
\[ \LR(A) = O(\log N \log M).\] 
\end{conjecture}

\paragraph{Upper bounds by a function of VC dimension?}

Perhaps the most intriguing open question in the study of list replicability is whether $\LR(\cC)$ can be upper bounded by a function of VC dimension. Originally posed in \cite{chase2023replicabilitystabilitylearning}, this question was resolved for extremal concept classes in \cite{blondal2026simplicial}, where it was shown that $\LR(\cC) = \Theta(\VCdim(\cC))$. However, the question remains open for arbitrary finite concept classes.

Moreover, our result that $\Ind_{\ZZ_2}(\cC) = O(\LR(\cC))$ suggests a natural addition to the question:

\begin{problem}\label{problem:1}
Can $\Ind_{\Z_2}(\cC)$ or $\LR(\cC)$ be bounded by a function of $\VCdim(\cC)$?
\end{problem}

Note that a positive answer to \cref{problem:1} would also imply $\coIndZ(\cC)$ is bounded by a function of $\VCdim(\cC)$. This question has already been asked in \cite{chornomaz2025spherical} as it has important implications for the learnability of \emph{large-margin halfspaces}.

\paragraph{Embeddings.}

One potential approach to resolving \cref{problem:1} is to embed the concept class $\cC$ into an extremal or intersection-closed concept class while keeping the VC dimension small. Recall that, in both settings, $\IndZ$ and $\LR$ are bounded above linearly by $\VCdim$. The problem of embedding into extremal concept classes has been widely studied and remains open \cite{moran2016labeled, chase2024dual}.

The analogous problem for intersection-closed classes was studied by Hanneke in \cite{hanneke2024star} through the minimum star number $s_{\min}$. Hanneke showed that if $\VCdim(\cC)=1$, then $s_{\min}(\cC)=1$; equivalently, every concept class of VC dimension $1$ can be embedded into a generalized intersection-closed class of VC dimension $1$. On the other hand, Hanneke also observed that there exists a family of finite concept classes with VC dimension $3$ and arbitrarily large minimum star number $s_{\min}$. This implies that, in general, one cannot embed an arbitrary concept class into an intersection-closed class without incurring a blow-up in VC dimension.

Hanneke asked what happens in the case $\VCdim(\cC)=2$. We show that $s_{\min}$ can be arbitrarily large in this case as well.

\begin{theorem}
For every $k\ge 2$, there is a finite total concept class $C_k$ with
\[
\VCdim(C_k)=2
\qquad\text{and}\qquad
s_{\min}(C_k)\ge k .
\]
\end{theorem}

\begin{proof}
   Let $X_k=[k]\times [k]$, and write $R_i=\{i\}\times [k]$ for the $i$-th row. For convenience, we will treat concepts as subsets of $\cX$, that is $c(x)=1$ if $x\in c$ and $c(x)=-1$ otherwise. Define
\[
C_k
=
\{\varnothing\}
\cup
\{R_i:i\in[k]\}
\cup
\{R_i\setminus\{x\}:i\in[k],\ x\in R_i\}.
\] 
First, $\VCdim(C_k)\ge 2$. Indeed, if $x,y\in R_i$ are distinct, then $\varnothing, R_i, R_i\setminus\{x\},R_i\setminus\{y\}$
realize all four labellings on $\{x,y\}$. Also, since any concept can only have $1$'s on one row $R_i$, it is not hard to see $\VCdim(C_k)\leq 2$. We conclude $\VCdim(\cC_k)=2$.

Now fix any center function $h:X_k\to\{-1,1\}$. We show that $s_h(C_k)\ge k$. 

If every row $R_i$ contains some point $x_i$ with $h(x_i)=-1$, let
\[
S=\{x_1,\dots,x_k\}.
\]
Then $\varnothing$ agrees with $h$ on $S$, and for each $i\in[k]$, the concept $R_i$ flips exactly the point $x_i$ on $S$. Hence $S$ is a star of size $k$ centered at $h$.

Otherwise, some row $R_i$ is entirely labeled $1$ by $h$. Take $S=R_i$. Then the concept $R_i$ agrees with $h$ on $S$, and for each $x\in R_i$, the concept $R_i\setminus\{x\}$ flips exactly $x$ on $S$. Hence $S$ is again a star of size $k$ centered at $h$.

Thus $s_h(\cC_k)\ge k$ for every $h$, so $s_{\min}(\cC_k)\ge k.$
\end{proof}

\section{Preliminaries}\label{sec:def_and_prop}
In this section, we collect all topological and learning-theoretic tools that we use. We first introduce basic notions from $\bZ_2$-equivariant topology, which allows us to formally define the $\bZ_2$-index of a sign matrix and its corresponding concept class. Afterwards, we go through some fundamental learning theoretical definitions and results, including list replicability and its topological and algorithmic interpretations.

\subsection{The $\Z_2$-topological framework}
\begin{definition}[Simplicial Complex]
An \emph{abstract simplicial complex} $\cK$ on a vertex set $V$ is a collection 
of finite subsets of $V$ that are closed under taking subsets. Its elements are called 
\emph{simplices} and its \emph{dimension} is $\max_{\sigma \in \cK} |\sigma| - 1$.
The \emph{geometric realization} of $\cK$ is the topological space
\begin{align*}
    \|\cK\| \coloneqq \left\{ x \in \mathbb{R}^V : x = \sum_{v \in V} \lambda_v e_v,\ 
    \lambda_v \geq 0,\ \sum_{v \in V} \lambda_v = 1,\ \mathrm{supp}(x) \in \cK \right\},
\end{align*}
where $e_v$ denotes the standard basis vector indexed by $v$ and $\mathrm{supp}(x) = \{v : \lambda_v > 0\}$.
\end{definition}

\begin{definition}[Free $\bZ_2$-Simplicial Complex]
    A \emph{free $\bZ_2$-simplicial complex} is a pair $(\cK,\nu)$ where $\cK$ is an abstract simplicial complex and $\nu$ is a \emph{fixed-point-free simplicial involution}, i.e., a
    simplicial map $\nu \colon \cK \to \cK$ satisfying $\nu \circ \nu = \mathrm{id}_\cK$ and $\nu(\sigma) \neq \sigma$
    for every $\sigma \in \cK$.
\end{definition}

\begin{definition}[Free $\bZ_2$-Space]
A \emph{free $\bZ_2$-space} is a pair $(X,\nu)$ where $X$ is a topological space and $\nu$ is a continuous \emph{fixed-point-free involution}, i.e., a continuous map $\nu \colon X \to X$ satisfying $\nu \circ \nu = \mathrm{id}_X$ and $\nu(x) \neq x$ for every $x \in X$.
\end{definition}

The geometric realization of a free $\bZ_2$-simplicial complex is naturally a free $\bZ_2$-space.
Indeed, if $\cK$ is a free $\bZ_2$-simplicial complex with involution $\nu$, the map
\begin{align*}
    \sum_{v \in V} \lambda_v e_v \mapsto \sum_{v \in V} \lambda_v e_{\nu(v)}
\end{align*}
is an involution on $\|\cK\|$.

\begin{example}
The $n$-dimensional sphere $\bS^n = \set{x \in \RR^{n+1} : \norm{x} = 1}$ equipped with the canonical  involution $\nu(x) = -x$ is a free $\bZ_2$-space.
\end{example}

\begin{definition}[Antipodal Map]
Let $(X,\nu)$ and $(Y,\omega)$ be free $\bZ_2$-spaces. We say a map $f \colon X \rightarrow Y$ is \emph{antipodal} if it commutes with the involutions on each space, that is $f\circ \nu=\omega \circ f$. We write $f \colon X \amap Y$ to denote that $f$ is an antipodal map between $X$ and $Y$.
\end{definition}

The antipodal terminology comes from the fact that $f$ cannot have $f(x) = f(\nu(x))$ for any $x\in X$ due to the involutions $\nu$ and $\omega$ being fixed-point-free.

\begin{definition}[$\bZ_2$-index]
The \emph{$\bZ_2$-index} of a free $\bZ_2$-space $(X,\nu)$ is the minimum integer $n$ such that there is an antipodal map from $X$ into the $n$-dimensional sphere.
\[
    \operatorname{Ind}\limits_{\mathbb{Z}_2}(X)
    \coloneqq \min\set{n \in \N : \exists f\colon X \amap \bS^n}.
\]
\end{definition}

It is also natural to instead consider the smallest sphere that maps antipodally into $X$, giving us the notion of $\bZ_2$-coindex.

\begin{definition}[$\bZ_2$-coindex]
The $\bZ_2$-coindex of a $\bZ_2$-space $X$ is the maximum $n$ such that there is an antipodal map from $\bS^n$ into $X$.
\[
    \operatorname{\coInd}\limits_{\mathbb{Z}_2}(X) \coloneqq \max\set{n \in \N : \exists f \colon \bS^n \amap X}.
\]
\end{definition}

The following is a generalization of the classical Borsuk--Ulam Theorem.

\begin{theorem}[Borsuk--Ulam]
Let $(X,\nu)$ be a free $\bZ_2$-space. Then
$\operatorname{\coInd}\limits_{\mathbb{Z}_2}(X) \leq
\operatorname{\Ind}\limits_{\mathbb{Z}_2}(X)$.
\end{theorem}

We refer the reader to Jiří Matoušek's excellent textbook on topological methods in combinatorics for more details on $\bZ_2$-spaces and their applications to combinatorics \citep[Section 5]{matousek2003borsuk}.

\subsection{The $\bZ_2$-index of sign matrices and concept classes}

To prove our results in full generality, we will work in the broader framework of partial sign matrices and partial concept classes.

If $A \in \set{1,-1,\star}^{M \times N}$ is a partial sign matrix, the \emph{sign-rank} of $A$, denoted $\signrank(A)$, is the minimum rank of a real matrix $B$ that captures the sign patterns of $A$. That is
\[
    \sign(B_{ij}) = A_{ij} \qquad \text{for every $i \in [M]$, $j \in [N]$ with $A_{ij}\neq \star$}.
\]

\paragraph{Partial concept classes.}

Partial sign matrices can also be interpreted as \emph{partial concept classes}. Given a domain $\cX$, a partial concept class $\cC \subseteq \set{+1,-1,\star}^{\cX}$ is a family of functions, called \emph{concepts}, mapping $\cX \to \set{+1,-1,\star}$. These are fundamental objects in learning theory because they encode bias, which is a prerequisite for any meaningful definition of learning.

The partial concept class of a partial sign matrix $A \in \set{1,-1,\star}^{M \times N}$ is a family $\cC_A \subseteq \set{+1,-1,\star}^{[N]}$ with (partial) concepts given by rows of $A$. That is,
\[
    c_i \colon j \mapsto A_{ij}.
\]

\paragraph{Space of realizable distributions.}

We say that a distribution $\mu\sim \cX\times\set{\pm1}$ is \emph{realizable} by a partial concept class $\cC$ if there exists a partial concept $c$ such that $c(x)=b$ for each $(x,b)$ in the support of $\mu$. We denote the set of all realizable distributions $\mu$ by
\begin{equation*}
\Delta_{\cC}\coloneq \set{\mu : \mu\;\text{realizable by}\; \cC }.    
\end{equation*}
When equipped with \emph{the total variation} (\text{TV}) \emph{distance}, $\Delta_{\cC}$
forms a metric space. $\Delta_{\cC}$ has a natural geometric realization as a subset of the $\ell_1$-sphere.

\begin{equation}
\label{eq:definition_D_C}\Delta_\cC \coloneqq \Set{\mu \in \mathbb{R}^\cX :~\norm{\mu}_1=1 \text{ and } \exists c \in \cC \text{ with } c(x)=b\mu(x) \ \forall (x,b) \in \supp(\mu)} \subseteq \RR^\cX.
\end{equation}

Note that each partial concept $c\in \cC$ corresponds to a simplex $\sigma_{c}\in \Delta_{\cC}$:
\[
    \sigma_{c}=\conv(\set{b \times e_x\;\colon\; x\in \cX,c(x)=b})
\]
where $\set{e_x}_{x\in\cX}$ is the standard basis of $\mathbb{R}^{\cX}$.

\begin{figure}[H]
\begin{center}

%%%% Delta_C in 2D %%%%
\begin{tikzpicture}[scale=0.4]
    \coordinate (LL) at (-4, 0);
    \coordinate (TL) at (-1, 3);
    \coordinate (TR) at (5, 3);
    \coordinate (BR) at (5, -3);
    \coordinate (BL) at (-1, -3);
    \coordinate (M) at (2, 0);

    \filldraw[thick, rounded corners=0.3mm, fill=gray!10]
        (TR) -- (TL) -- (M) -- (TR)
        (TR) -- (BR) -- (M) -- (TR)
        (BR) -- (BL) -- (M) -- (BR)
        (BL) -- (TL) -- (M) -- (BL)
        (TL) -- (LL) -- (BL) -- (TL);

    \filldraw[shift only](TL) circle (2pt) node[anchor=south]{};
    \filldraw[shift only] (TR) circle (2pt) node[anchor=south]{};
    \filldraw[shift only] (BR) circle (2pt) node[anchor=north]{};
    \filldraw[shift only] (BL) circle (2pt) node[anchor=north]{};
    \filldraw[shift only] (LL) circle (2pt) node[anchor=east]{};
    
    \filldraw[shift only] (M) circle (2pt);
    \node at (2, 2) {\texttt{+-+}};
    \node at (2, -2) {\texttt{-++}};
    \node at (3.6, 0) {\texttt{--+}};
    \node at (0.4, 0) {\texttt{+++}};
    \node at (-2.4, 0) {\texttt{++-}};

    \node[font=\large] at (-4, 3) {$\Delta_{\cC}$};
\end{tikzpicture}
\hspace{4em}
%
%
%%%% Delta_C in 3D %%%%
\begin{tikzpicture}[x={(1cm,0cm)}, y={(0cm,1cm)}, z={(0.1cm,0.4cm)}, scale=0.5]
    \coordinate (A) at (-1, -1, -1);
    \coordinate (AB) at (-1.5, 0, -1.5);
    \coordinate (B) at (-1, 1, -1);
    \coordinate (BC) at (-1.5, 1.5, 0);
    \coordinate (C) at (-1, 1, 1);
    \coordinate (CD) at (0, 1.5, 1.5);
    \coordinate (D) at (1, 1, 1);
    \coordinate (DE) at (1.5, 1.5, 0);
    \coordinate (E) at (1, 1, -1);
    \coordinate (EB) at (0, 1.5, -1.5);

    \draw[very thick, rounded corners=0.1mm, fill=black, fill opacity=0.1] (0, 3, 0) -- (-3, 0, 0) -- (0, 0, -3) -- cycle;
    \draw[very thick, rounded corners=0.1mm, fill=black, fill opacity=0.1] (0, -3, 0) -- (-3, 0, 0) -- (0, 0, -3) -- cycle;
    \draw[very thick, rounded corners=0.1mm, fill=black, fill opacity=0.1] (0, 3, 0) -- (3, 0, 0) -- (0, 0, -3) -- cycle;
    \draw[rounded corners=0.1mm, fill=black, fill opacity=0.1] (0, 3, 0) -- (-3, 0, 0) -- (0, 0, 3) -- cycle;
    \draw[rounded corners=0.1mm, fill=black, fill opacity=0.1] (0, 3, 0) -- (3, 0, 0) -- (0, 0, 3) -- cycle;

    \node[font=\large] at (3.5 , 2.5, 0) {$\Delta_{\cC} \subset \RR^3$};
\end{tikzpicture}
\end{center}

\caption{Two views of the simplicial complex $\Delta_\cC$ for $\cC = \set{\texttt{++-}, \texttt{+++}, \texttt{+-+}, \texttt{--+}, \texttt{-++}}$.}
\label{fig:space_of_realizable_dsitributions}
\end{figure}

The following properties of $\Delta_{\cC}$ are easy to deduce and can also be found in \cite{chornomaz2025spherical,blondal2026simplicial,frick2026signrank}.

\begin{itemize}
    \item $\Delta_{\cC}$ is a finite compact simplicial complex.
    \item $\Delta_{\cC}$ has vertex set contained in $\cX\times\set{\pm 1}$ and is a subcomplex of the cross-polytope boundary $\Delta_{\set{\pm1}^{\cX}}$.
    \item Each maximal simplex of $\Delta_{\cC}$ equals $\sigma_c$ for some $c\in \cC$.
    \item If $\cC^{\pm}=\cC\cup-\cC$, then $\Delta_{\cC^{\pm}}$ is a free $\ZZ_2$-space with a natural $\ZZ_2$-action given by negating labels: $\mu \mapsto -\mu$, where $-\mu$ assigns mass $\mu(x,b)$ to $(x,-b)$.
\end{itemize}

\begin{definition}[Index/Coindex]
Let $A$ be a partial sign matrix and let $\cC_A$ be its associated partial concept class. Then we define
\[\IndZ(A)\coloneq\IndZ(\cC_A)\coloneq\IndZ(\Delta_{\cC_{A}^{\pm}})\]
\[\coIndZ(A)\coloneq\coIndZ(\cC_A)\coloneq\coIndZ(\Delta_{\cC_{A}^{\pm}})\]
\end{definition}

\subsection{List replicability}

The motivation for our discussion of the concept class of a sign matrix in the previous section is that this object allows us to leverage existing \emph{list replicability} techniques from learning theory.

A \emph{learning rule} is a (possibly randomized) function $\cA$ that maps any \emph{sample}
$S \in \bigcup_{n=0}^\infty (\cX \times \Set{\pm 1})^n$ 
to a \emph{hypothesis} $\cA(S) \in \Set{\pm 1}^{\cX}$. The error or \emph{population loss} of a hypothesis with respect to a distribution $\mu$ over $\cX\times\set{\pm1}$ is measured as
\[
    \loss_\mu(h) = \Pr_{(x,y) \sim \mu}[h(x) \neq y].
\]

\begin{definition}[List Replicability, \cite{chase2023replicabilitystabilitylearning, dixon2023listandcertificate}]\label{def:list_replicability}
A learning rule $\cA$ is an
$(\epsilon,L)$-\emph{list-replicable learner} for a partial concept class $\cC$ if for every $\delta>0$ there exists a sample complexity $n \coloneq n(\delta)$ such that the following holds.
For every distribution $\mu$ realizable by $\cC$, there exists a list of hypotheses $h_1,\ldots,h_L \in \Set{\pm 1}^\cX$ such that 
\[
    \loss_\mu(h_i) \le \epsilon ~\forall i
    ~\text{ and } ~
    \Pr_{S \sim \mu^n} [\cA(S) \in \Set{h_1,\ldots,h_L}] \geq  1-\delta.
\]
The $\epsilon$-\emph{list replicability number} of $\cC$ is
\[
    \LR(\cC,\epsilon) \coloneqq \min\{L : \exists (\epsilon,L)\text{-list-replicable learner for }\cC\},
\]
with $\LR(\cC,\epsilon)=\infty$ if none exists.
The \emph{list replicability number} of $\cC$ is 
\[
    \LR(\cC) \coloneqq \sup_{\epsilon>0} \LR(\cC,\epsilon).
\]
We say $\cC$ is \emph{list-replicable} if $\LR(\cC)<\infty$.
\end{definition}

One of the advantages of list replicability is that this purely learning theoretic definition also admits a topological description via closed covers of $\Delta_{\cC}$. Recall that for each $h\in\set{\pm 1}^{\cX}$ we defined
\[\sigma_{h}=\conv(\set{b \times e_x\;\colon\; x\in \cX,h(x)=b})\]
where $\set{e_x}_{x\in\cX}$ is the standard basis of $\mathbb{R}^{\cX}$.

Now for each hypothesis $h\in \set{\pm1}^{\cX}$ let $B_{\epsilon}(\sigma_h)$ denote the set of all points in $\Delta_{\set{\pm 1}^{\cX}}$ (equivalently the cross-polytope boundary) that are within total variation distance less than $\epsilon$ from the simplex $\sigma_h$. Denote by $\overline{B_{\epsilon}}(\sigma_h)$ the respective closure. We call the sets $B_{\epsilon}(\sigma_h)$ and $\overline{B_{\epsilon}}(\sigma_h)$ the open and closed \emph{$\epsilon$-loss} sets of $h$ because they contain all distributions $\mu\in \Delta_{\cC}$ for which $\loss_{\mu}(h)<\epsilon$ and $\loss_{\mu}(h)\le\epsilon$ respectively.

A recent line of work in \cite{chase2023replicabilitystabilitylearning, localborsukulam, blondal2026simplicial} has shown a correspondence between list-replicable learners for $\cC$ and closed/open covers of $\Delta_\cC$. In both cases, list size corresponds to the \emph{overlap degree} of the cover, which we define as the maximum number of sets in the cover with a common non-empty intersection.

\begin{theorem}[{\cite[Corollary 23]{localborsukulam}, \cite[Theorem A]{blondal2026simplicial}}]
\label{thm:closed-cover-lr}
Let $\cC\subseteq\{\pm1,\star\}^\cX$ be a finite partial concept class and
let $\epsilon>0$.
\begin{itemize}
    \item \textnormal{(Closed cover characterization)} The $\epsilon$-list replicability number $\LR(\cC,\epsilon)$ equals the minimum integer $L$ for which there exists a closed cover $\cF=\{F_h:h\in \set{\pm 1}^{\cX}\}$ of \(\Delta_\cC\) with overlap degree at most $L$ and $F_h\subseteq \overline B_\varepsilon(\sigma_h)$ for every $h\in\set{\pm1}^{\cX}$.
    \item \textnormal{(Open cover characterization)} If $L\coloneq \LR(\cC,\epsilon)$, then for every $\eta>\epsilon$ there exists an open cover $\cU=\set{U_h:h\in\set{\pm1}^{\cX}}$ of $\Delta_\cC$ with overlap degree at most $L$ such that $U_h\subseteq B_\eta(\sigma_h)$ for every $h\in\set{\pm1}^{\cX}$. Conversely, if for some $\eta>0$ and $L\in\N$ there exists an open cover $\cU=\set{U_h:h\in\set{\pm1}^{\cX}}$ of $\Delta_\cC$ with overlap degree at most $L$ such that $U_h\subseteq B_\eta(\sigma_h)$ for every $h\in\set{\pm1}^{\cX}$, then there exists some $\epsilon<\eta$ such that $\LR(\cC,\epsilon)\le L$.

\end{itemize}
\end{theorem}

The closed cover characterization of list replicability was shown in \cite{chase2023replicabilitystabilitylearning} for total concept classes, but the argument immediately generalizes to partial concept classes. The open cover characterization is closely related, and was studied in a stronger form as \emph{simplicial covering dimension} in \cite{blondal2026simplicial}.

\section{Separating sign-rank and index}\label{sec:separate_sgn_ind}

Recall from \cref{sec:technical_overview} that our separation of sign-rank and index in \cref{thm:main} leverages the family of incidence matrices $B_q$ for the finite projective planes $\PG(2,q)$. We are guaranteed the polynomial growth of $\signrank(B_q)$ from \cref{thm:poly_growth_sgn_rnk}, so this section will collect the proofs of \cref{thm:index_vs_LR} and \cref{thm:main} to bound $\IndZ(B_q)$ by a constant.

\subsection{List replicability number upper bounds index}

In this section, we prove our main technical result \cref{thm:index_vs_LR}, stated here in a slightly stronger form.

\begin{theorem}\label{thm:index_vs_LR_eps}
For every (partial) concept class $\cC$ and error parameter $\eps \in (0,\frac{1}{2})$,
\[
\Ind_{\ZZ_2}(\cC) \;\le\; 2\,\LR(\cC, \eps) - 1.
\]
\end{theorem}

The argument begins by producing a closed cover of $\Delta_\cC$ with small overlap degree, as guaranteed by \cref{thm:closed-cover-lr}. We use basic topology to extend that cover to an open cover of $\Delta_{\cC^\pm}$ which respects the $\bZ_2$ structure of that space. This step is necessary because $\IndZ(\cC)$ is by definition a property of $\Delta_{\cC^\pm}$ rather than $\Delta_\cC$.

From here, we map $\Delta_{\cC^\pm}$ into low-dimensional Euclidean space, which we in turn project into a sphere.
\begin{align*}
    \Delta_{\cC^\pm} \amap \RR^{2L} \setminus \{0\} \amap \bS^{2L-1}.
\end{align*}
By definition, the composition of these maps witnesses an upper bound on $\IndZ(\cC)$. The construction of the map $\Delta_{\cC^\pm} \amap \RR^{2L}$ was discovered by considering the nerve complex associated with the open cover of $\Delta_{\cC^\pm}$. This is a simplicial complex that encodes intersection data, making it a natural candidate to analyze overlap degree.

\begin{proof}[Proof of \cref{thm:index_vs_LR_eps}]
For convenience, set $L \coloneqq \LR(\cC,\eps)$. Let $\cA$ be an $(\eps,L)$-list-replicable learner for $\cC$. By the closed cover characterization of list replicability in \cref{thm:closed-cover-lr}, $\cA$ induces a closed cover of $\Delta_\cC$ given by $\cF = \set{F_h : h \in \bcube}$, where
\[
    \text{the overlap degree of $\cF$ is $L$} \qquad \text{and} \qquad \text{$F_h \subseteq \overline{B_\eps}(\sigma_h)$ for all $h \in \bcube$}.
\]

Recall that $\IndZ(\cC)$ is defined to be the $\bZ_2$-index of $\Delta_{\cC^\pm}$. Since $\Delta_\cC$ is a closed subspace of $\Delta_{\cC^\pm}$, each set $F_h$ remains closed under the natural inclusion $\Delta_\cC \hookrightarrow \Delta_{\cC^\pm}$. Hence, the family
\[
\cF'\coloneq \set{F'_h\,\colon\;F'_h=F_h\cup(-F_{-h})\;\text{for some }h\in\set{\pm1}^{\cX}}
\]
is a closed cover of $\Delta_{\cC^\pm}$ with overlap degree at most $2L$ satisfying $F'_h\subseteq\overline{B_{\epsilon}}(\sigma_h)$ and $F'_h=-F'_{-h}$ for all $h$.

Now pick any $\epsilon_0\in(\epsilon,1/2)$. Because $\Delta_{\cC^\pm}$ is compact, we can extend each closed set $F'_h$ to an open set $U_h\subseteq B_{\epsilon_0}(\sigma_h)$ without increasing the overlap degree of the family and while preserving the antipodal symmetry $U_h=-U_{-h}$\footnote{Let $d\colon \Delta_{\cC^\pm}\rightarrow [0,1]$ be a function measuring the total variation distance $d(x)$ from a point $x\in \Delta_{\cC^\pm}$ to the $(2L+1)$-st closest set $F'_h$. By compactness, this function achieves a strictly positive minimum $\beta$. Take $U_h\coloneq \cup_{x\in F'_h}B_{\beta/2}(x)$.}. Then $\cU\coloneq \set{U_h : h\in\set{\pm1}^{\cX}}$ is a finite open cover of $\Delta_{\cC^\pm}$ with overlap degree at most $2L$ such that $U_h \subseteq B_{\epsilon_0}(\sigma_h)$ and $U_h=-U_{-h}$ for all $h$.

Since $\Delta_{\cC^\pm}$ is a compact metric space, there exists a partition of unity $\set{\phi_h}_h$ subordinate to the open cover $\cU$ \cite[Theorem 2.13]{rudin1987real}.
That is, there exists a collection of maps $\phi_h \colon \Delta_{\cC^\pm} \to [0, 1]$ such that
\begin{align*}
    \supp \phi_h \subseteq U_h \quad \text{ and } \quad 
    \sum_{h} \phi_h \equiv 1.
\end{align*}
We can make an antipodally symmetric version of $\set{\phi_h}_h$ by setting
\begin{align*}
    \phi^*_h(x) &\coloneqq \frac{1}{2} (\phi_h(x) + \phi_{-h}(-x)).
\end{align*}
Note that the antipodal symmetry $U_h=-U_{-h}$ guarantees that $\supp \phi_h^*(x)\subseteq U_h$.

Next, let $\set{p_h}_h \subseteq \RR^{2L}$ be a set of points in general position such that $p_{-h} = -p_h$. That is, a subset $P \subset \set{p_h}$ of size $2L$ contains zero in its convex span if and only if $P$ contains a pair $\set{p_h, p_{-h}}$ for some $h$.

Using the antipodally symmetric partition of unity $\phi^*$ and the points
$\{p_h\}_h$, we define an antipodal map $f:\Delta_{\cC^\pm} \amap \RR^{2L}$ by
\[
    f(x)=\sum_h \phi^*_h(x)p_h.
\]
We claim that $0\notin \operatorname{im}(f)$. Indeed, for every
$x\in\Delta_{\cC^\pm}$, the overlap degree of $\cU$ ensures $\phi^*_h(x)$ is
nonzero for at most $2L$ hypotheses $h$. Additionally, $\eps_0 < 1/2$ implies that no $U_h$ contains both $x$ and $-x$, so for every $x$ at least one of $\phi_h^*(x)$ and $\phi_{-h}^*(x)$ must be zero. Thus, the coefficients of $p_h$ and $p_{-h}$ cannot both be positive. It follows that $f(x)$ is a convex combination of at most $2L$ points of $P$, which does not contain both $p_h$ and $p_{-h}$. Since $P$ is in general position, $0\notin \operatorname{im}(f)$.

We may therefore project the image of $f$ into the sphere $\bbS^{2L-1} \subset \RR^{2L}$ to get a map
\begin{align*}
    \frac{f}{\norm{f}_2}: \Delta_{\cC^{\pm}} \amap \bS^{2L-1}.
\end{align*}
Exhibiting such an antipodal map into the sphere shows that $\Ind_{\bZ_2}(\Delta_{\cC^{\pm}}) \leq 2L-1$ by definition of index.
\end{proof}

\subsection{A list-replicable algorithm for $\PG(2,q)$}

As defined in \cref{def:PZq2}, recall the incidence geometry $\PG(2,q) = (\cP, \cL)$.
Let $\cC_q$ denote the corresponding concept class
\begin{align*}
    \cC_q &\coloneqq \set{c_\ell: \ell \in \cL}\\
    c_\ell : \cP &\to \set{\pm 1}\\
    p &\mapsto \begin{cases}
        1 &\text{if } p \in \ell\\
        -1 &\text{otherwise.}
    \end{cases}
\end{align*}
We exhibit a list-replicable algorithm for $\cC_q$.
By sampling random points $x \in \cP$ and whether they lie on a line $\ell$, the learner's goal is to output a predictor of whether future points lie on the same line.

\begin{theorem}\label{thm:lr_of_PZq2}The list replicability number of $\cC_q$ satisfies $\LR(\cC_q) \leq 3$.
\end{theorem}
\begin{proof} We extend the concept class $\cC_q$ to a larger hypothesis class $\cH_q$
    \begin{align*}
        \cH_q = \set{c_\ell : \ell \in \cL} \cup \set{c_p : p \in \cP} \cup \set{c_{-1}},
    \end{align*}
    where $c_{-1}$ is the all-minus hypothesis, and $c_p$ is the indicator function for a single point $p$, evaluating to $1$ on $p$ and $-1$ everywhere else.

    Since $\cH_q$ is finite, we have by the union bound and Hoeffding's inequality\footnote{Let $c\in \RR$ and let $\bm{x}_1,\ldots, \bm{x}_n$ be independent random variables with $\bm{x}_i\in [-c,c]$ and $\E[\bm{x}_i]=0$. For any $t>0$,
\[
\Pr \left[\left|\sum_{i=1}^n \bm{x}_i\right|\geq t \right]\leq 2e^{-\frac{t^2}{2nc^2}}. 
\]} that there exists an $n(\epsilon, \delta)$ such that for any $\cD \in \Delta_{\cC_q}$, we can estimate the population loss of all hypotheses in $\cH_q$ simultaneously:
    \begin{equation}\label{eq:uniform_convergence}
        \Pr_{S \sim \mathcal{D}^{n}}
        \left[
            \sup_{h \in \cH_q} |\loss_{S}(h) - \loss_{\cD}(h)| \leq \frac{\epsilon}{8}
        \right] \geq 1 - \delta.
    \end{equation}

    Next, we define the learning rule $\cA$ (\cref{alg:A}).
    For any $\epsilon, \delta$, let its sample size be $n(\epsilon, \delta)$ so as to satisfy \eqref{eq:uniform_convergence}.

    \begin{algorithm}[ht]
    \caption{\label{alg:A}The learning rule $\cA$}
    \begin{algorithmic}[1]
        \STATE Sample $S = (S_1, \dots, S_n) \sim \cD^n$.
        \FOR{$p \in \cP$}
            \STATE Let $\hat\cD(p) = \frac{1}{n}|\set{i \in [n] : S_i = (p, 1)}|$.
        \ENDFOR
        \STATE Let $P_0 \gets \set{p \in \cP : \hat\cD(p) > 0}$
        \STATE Let $P_1 \gets \set{p \in \cP : \hat\cD(p) > 7\epsilon / 8}$
        \IF{$|P_0| \geq 2$}
            \STATE Output $c_\ell$, where $\ell$ is the unique line containing all points in $P_0$.
        \ELSIF{$|P_1| = 1$}
            \STATE Output $c_p$, where $P_1 = \set{p}$.
        \ELSE
            \STATE Output $c_{-1}$.
        \ENDIF
    \end{algorithmic}
    \end{algorithm}

    We start by confirming that this algorithm is a PAC learner.

  By \eqref{eq:uniform_convergence}, with probability $\geq 1 - \delta$,
    \begin{align*}
        \sup_{h \in \cH_q} |\loss_{S}(h) - \loss_{\cD}(h)| \leq \frac{\epsilon}{8}.
    \end{align*}
    Denote this event by $E$.

    \textbf{Case 1.}
    Whenever $|P_0| \geq 2$, the hypothesis output is correct and has no loss.

    \textbf{Case 2.}
    If $P_0 = P_1 = \set{p}$, then $\cA(S) = c_p$ and $\loss_{S}(c_p) = 0$. Hence, given $E$, we have
    \[
        \loss_{\cD}(c_p) \leq \frac{\epsilon}{8}.
    \]

    \textbf{Case 3.}
    If $|P_0| = 0$, then once again, the hypothesis output has no empirical loss, so given $E$,
    \begin{align*}
        \loss_{\cD}(c_{-1}) \leq \frac{\epsilon}{8}.
    \end{align*}

    \textbf{Case 4.}
    Finally, if $P_0 = \set{p}$ but $P_1 = \emptyset$, then given $E$, we have
    \begin{align*}
        \loss_{\cD}(c_{-1}) \leq \loss_{S}(c_{-1}) + \frac{\epsilon}{8}
        \leq \frac{7\epsilon}{8} + \frac{\epsilon}{8} = \epsilon.
    \end{align*}

Next, we show the list replicability of $\cA$.

    \begin{claim}\label{claim:LR}
        Given $E$, the algorithm $\cA$ cannot output $c_p$ and $c_{p'}$ for two distinct points $p \neq p'$.
    \end{claim}

\begin{proof}
    Suppose that $\cA$ outputs $c_p$ on some sample satisfying $E$. Then $P_0=P_1=\set{p}$, so $\loss_S(c_{-1})=\hat\cD(p)>7\epsilon/8$ and $\loss_S(c_p)=0$. By $E$,
\[
    \loss_{\cD}(c_{-1})>\frac{7\epsilon}{8}-\frac{\epsilon}{8}=\frac{3\epsilon}{4},
    \qquad
    \loss_{\cD}(c_p)\leq \frac{\epsilon}{8}.
\]
Moreover, since $c_p$ is output, the point $p$ appears with label $1$, so $p$ lies on any line realizing $\cD$. Hence
\[
    \cD(p)=\Pr_{(x,y)\sim\cD}[x=p]
    =
    |\loss_{\cD}(c_{-1})-\loss_{\cD}(c_p)|
    >
    \frac{5\epsilon}{8}.
\]
Now let $S'$ be any sample satisfying $E$. Again using $E$ for the two hypotheses $c_{-1}$ and $c_p$,
\[
    \hat\cD_{S'}(p)
    =
    |\loss_{S'}(c_{-1})-\loss_{S'}(c_p)|
    \geq
    |\loss_{\cD}(c_{-1})-\loss_{\cD}(c_p)|-\frac{2\epsilon}{8}
    >
    \frac{3\epsilon}{8}.
\]
Thus, every sample satisfying $E$ contains at least one copy of $(p,1)$.

Therefore, if both $c_p$ and $c_{p'}$ could be output on samples satisfying $E$, then every sample satisfying $E$ would contain both $(p,1)$ and $(p',1)$. But then $|P_0|\geq 2$, and the algorithm would output the unique line through $p$ and $p'$, not a point hypothesis. This contradiction proves the claim.
\end{proof}

By \cref{claim:LR}, we see that when $E$ holds, $\cA$ has population loss at most $ \epsilon$, and outputs one of at most $3$ different hypotheses.
Therefore,
$\LR(\cC_q) \leq 3.$\end{proof}

\section{A height-based list replicability algorithm}\label{sec:height_vs_lr}

In \cite{frick2026signrank}, Frick, Hosseini, and Vasileuski used a combinatorial notion of the height of a simplicial complex to upper-bound the index of a concept class.
We show that this height is an upper bound for list replicability as well.
Recall from \cref{subsec:height_eluder} the definition of the junction closure $\cJ(\cC)$ of a partial concept class $\cC\subseteq \set{\pm1,\star}^\cX$:
\[
    \cJ(\cC)
    =
    \left\{
        \bigcap_{c \in S} c : \emptyset \neq S \subseteq \cC.
    \right\}
\]
Then, the height $\height(\cC)$ measures the length of the longest inclusion chain in $\cJ(\cC)$.
We restate \cref{thm:height_bounds_LR} for completeness.

\heightboundslr*

Frick, Hosseini, and Vasileuski applied their result to show that with high probability, random sign matrices have index $O(\log N)$.

\begin{theorem}[Random sign matrices have $O(\log N)$ height and index, \cite{frick2026signrank}]\label{thm:random_mat_height}
    Let $A \in \{\pm 1\}^{N \times N}$ be a sign matrix with entries sampled independently and uniformly at random.

    Then, there exists some constant $C > 0$ such that with probability $1 - o(1)$,
    \begin{align*}
        \height(A) \leq C \cdot \log N,
    \end{align*}
    and therefore
    \begin{align*}
        \IndZ(A) \leq 2C \cdot \log N.
    \end{align*}
\end{theorem}

\begin{corollary}[Random sign matrices have $O(\log N)$ list replicability]
    For $A \in \{\pm 1\}^{N \times N}$ with entries sampled independently and uniformly at random, with probability $1 - o(1)$,
    \begin{align*}
        \LR(A) \leq C \cdot \log N
    \end{align*}
    as well.
\end{corollary}
\begin{proof}
    The corollary follows directly from \cref{thm:random_mat_height,thm:height_bounds_LR}.
\end{proof}

To prove \cref{thm:height_bounds_LR}, we give an algorithm that, with high probability, will only ever output a hypothesis contained within a single chain of $\cJ(\cC)$.
Therefore, the list replicability of this algorithm will be naturally bounded by the height of $\cC$.

The algorithm is fairly similar to that described in \cref{thm:lr_of_PZq2}.
We estimate the error of every partial hypothesis in $\cJ(\cC)$, and pick the smallest one that has ``low error''.
In particular, our threshold for ``low error'' diminishes as the size of the hypothesis grows, so that we can guarantee that if any two partial hypotheses have low enough error, so does their junction.
This way, we prevent any antichains in our output set.

\begin{proof}[Proof of \cref{thm:height_bounds_LR}]
    Let the domain $\cX$ of $\cC$ have size $|\cX| = N$.
    
    Since $\cJ(\cC)$ is finite, by Hoeffding's inequality, for every $\epsilon, \delta > 0$, there exists an $n(\epsilon, \delta)$ such that for any $\cD \in \Delta_{\cC}$, we can estimate the population loss of all hypotheses in $\cJ(\cC)$ simultaneously:
    
    \begin{equation}\label{eq:uniform_convergence_2}
        \Pr_{S \sim \mathcal{D}^{n}}
        \left[
            \sup_{c \in \cJ(\cC)} |\loss_{S}(c) - \loss_{\cD}(c)| \leq \frac{\epsilon}{4^N}
        \right] \geq 1 - \delta.
    \end{equation}

    We define the learning rule $\cA$.
    For any $\epsilon, \delta$, let its sample size be $n(\epsilon, \delta)$ so as to satisfy \eqref{eq:uniform_convergence_2}.
    Denote by E the event that all losses are estimated within $\epsilon 4^{-N}$. Denote by $|c|$ the number of non-$\star$ points in the concept.

    \begin{algorithm}[H]
    \caption{\label{alg:height}The learning rule $\cA$}
    \begin{algorithmic}[1]
        \STATE Sample $S = (S_1, \dots, S_n) \sim \cD^n$.
        \FOR{$c \in \cJ(\cC)$, ordered from smallest to largest by size $|c|$}
            \IF{$\loss_S (c) \leq \epsilon \cdot 4^{-|c|}$}
                \STATE Output $c$, or any completion of $c$.
            \ENDIF
        \ENDFOR
        \STATE Output ERROR (the algorithm never reaches this state)
    \end{algorithmic}
    \end{algorithm}

    First of all, it is clear that $\cA$ is a PAC learner. If E holds, then
    \begin{align*}
        \loss_\cD(\cA(S)) \leq \frac{\epsilon}{4^N} + \loss_S(\cA(S)) \leq \frac{\epsilon}{4^N} + \frac{\epsilon}{4^{|\cA(S)|}} \leq \epsilon.
    \end{align*}
    Since $\cA(S)$ can be a partial hypothesis, outputting any completion of it will never make the error grow.

    \begin{claim}
        Let event E hold.
        Then, if $c_1$ and $c_2$ both satisfy the output requirement in line 3 of \cref{alg:height}, so does $c_1 \cap c_2$.
    \end{claim}
    \begin{proof}
        Without loss of generality, if $c_1 \preceq c_2$, then this holds trivially.
        Otherwise, $|c_1 \cap c_2| < \min\set{|c_1|, |c_2|}$.

        Note that $c_1 \cap c_2$ is only correct on a sample if both $c_1$ and $c_2$ are.
        Thus,
        \begin{align*}
            \loss_\cD(c_1 \cap c_2) 
            &= \Pr_{(x,y) \sim \cD}[c_1(x) \neq y \text{ or } c_2(x) \neq y]\\
            &\leq \Pr_{(x,y) \sim \cD}[c_1(x) \neq y]
            + \Pr_{(x,y) \sim \cD}[c_2(x) \neq y]\\
            &\leq \loss_\cD(c_1) + \loss_\cD(c_2)\\
            &\leq \loss_S(c_1) + \loss_S(c_2) + \frac{2\epsilon}{4^N}\\
            &\leq \frac{\epsilon}{4^{|c_1|}} + \frac{\epsilon}{4^{|c_2|}} + \frac{2\epsilon}{4^{N}}\\
            &\leq \frac{4\epsilon}{4^{\min\set{|c_1|, |c_2|}}}\\
            &\leq \frac{\epsilon}{4^{|c_1 \cap c_2|}}
        \end{align*}
        It is clear that $\cJ(\cC)$ is closed under taking junctions, so $c_1 \cap c_2 \in \cJ(\cC)$, and would be output first if both $c_1$ and $c_2$ fit the conditions.
    \end{proof}    
    As a result of this claim, so long as E holds, no two concepts in an antichain will be output.
    Thus, so long as $\cA$ doesn't output ERROR, when E holds, it will only output hypotheses from a single chain, so at most $\height(\cC)$ different hypotheses.

    Finally, $\cA$ will never output ERROR, for the distribution $\cD$ comes from $\Delta_\cC$, and thus some $c \in \cC \subseteq \cJ(\cC)$ has $0$ error on it.
    So if all else fails, that $c$ always fits the conditions to be output by $\cA$.
\end{proof}

Notice that this algorithm can, instead of outputting partial hypotheses, output any completion desired.
In particular, many of the partial hypotheses in $\cJ(\cC)$ may have common completions.
Therefore, by using specific completions, better list replicability bounds may be obtained.

\section{Height and eluder dimension}

We now prove a more general form of \cref{prop:height_eluder}, stated in the introduction, that the height parameter coincides with the eluder dimension up to an additive constant. We keep the notation from \cref{subsec:height_eluder}: $\cJ(\cC)$ is the junction closure of $\cC$, and $g\preceq h$ means that $g$ is a restriction of $h$. For a partial concept $h$, write $    \supp(h) := \{x\in\cX : h(x)\neq \star\}$.

\begin{definition}[Height relative to a base concept]
Fix a finite partial concept class $\cC\subseteq\{\pm1,\star\}^{\cX}$ and a base concept $c'\in\cC$. Define the height of $\cC$ relative to $c'$ by
\[
\operatorname{H}(\cC, c')
:=
\sup\left\{
    m:
    \exists h_1,\dots,h_m\in \cJ(\cC)
    \text{ such that }
    h_1\prec h_2\prec\cdots\prec h_m=c'
\right\}.
\]
\end{definition}

This rooted version recovers the height from the introduction:
\[
    \operatorname{H}(\cC)=\sup_{c'\in\cC}\operatorname{H}(\cC, c').
\]

Also, recall the definition of the eluder dimension. We extend the definition of \cite{li2022understanding} to partial concept classes.

\begin{definition}[\cite{li2022understanding}]\label{def:eluder}
Fix a finite partial concept class $\cC\subseteq\{\pm1,\star\}^{\cX}$ and a base concept $c'\in\cC$. The eluder dimension of
\(\mathcal C\) relative to \(c'\), denoted $\operatorname{Edim}(\cC,c'),$
is the largest $m$ such that there exist points $x_1,\ldots,x_m\in \supp(c')$ and concepts $
c_1,\ldots,c_m\in\mathcal C
$ such that, for every \(i\in[m]\),
\[
c_i(x_i)\neq c'(x_i),
\;\text{\textnormal{($c_i(x_i)=\star$ allowed)}}\]
while for every earlier point \(x_j\), \(j<i\),
\[
c_i(x_j)=c'(x_j).
\]
Define the eluder dimension of $\cC$ by $\operatorname{Edim}(C)\coloneq\sup_{c'\in\cC}\operatorname{Edim}(\cC,c')$.
\end{definition}

\begin{proposition}[Height equals eluder dimension]\label{prop:eluder}
For every finite partial concept class $\mathcal C\subseteq\{\pm1,\star\}^{\cX}$
and every base concept $c'\in \cC$,
\[
\operatorname{H}(\cC,c')
=
\operatorname{Edim}(\cC,c')+1.
\]
In particular,
\[
\operatorname{H}(\cC)
=
\operatorname{Edim}(\cC)+1.
\]
\end{proposition}
\begin{proof}
    We prove both inequalities.

First suppose $\operatorname{Edim}(\cC,c')= m$.
For \(i=0,1,\ldots,m\), define
\[
h_i
=
\bigcap
\left(
\{c'\}\cup \{c_{i+1},c_{i+2},\ldots,c_m\}
\right),
\]
and note that
\[
h_0\prec h_1\prec\cdots\prec h_m=c'.
\]
Therefore there is a strict chain of \(m+1\) partial concepts, so
\[
\operatorname{H}(\cC,c')\ge m+1.
\]

Conversely, suppose $
\operatorname{H}(\cC,c')= m.$
Then there exist $h_1\prec\cdots\prec h_m=c'$
in $\cJ(\cC)$. For each \(i=1,\ldots,m-1\), since \(h_{i}\prec h_{i+1}\), choose a point
\(x_i\in X\) such that
\[
h_{i}(x_i)=\star
\qquad\text{and}\qquad
h_{i+1}(x_i)\in\{\pm1\}.
\]
Because \(h_{i+1}\preceq c'\), we have $
h_{i+1}(x_i)=c'(x_i).$ Suppose $h_i=\bigcap_{c\in S_i} c$ for some $S_i\subseteq C$. There exists some $c_i\in S_i$ such that 
\[c_i(x_i)\neq c'(x_i)\quad \text{and}\quad c_i(x_j)=c'(x_j)\quad \text{for all }j<i.\]

Thus the domain points $x_1,\dots,x_{m-1}$ together with the concepts $c_1,\dots,c_{m-1}$ form an eluder sequence with base concept $c'$. Therefore
\[
\operatorname{Edim}(\cC,c')\ge m-1.
\]
Combining the two inequalities,
\[
\operatorname{H}(\mathcal C,c')
=
\operatorname{Edim}(\cC,c')+1.
\]\
\end{proof}

\section{Separation of coindex and list replicability}\label{sec:cor_and_connections}

In this section, we discuss \cref{cor:separation_LR_coInd}, which we restate below.

\LRvscoInd*

This result follows from combining our \cref{thm:index_vs_LR} with a result of \cite{frick2026signrank}. We have shown that $\LR(\cC)\geq (\IndZ(\cC)+1)/2$, and it is known that the projective spaces $\mathbb{RP}^{2N-1}$ are free $\bZ_2$-spaces separating $\IndZ$ and $\coIndZ$. To obtain that separation for $\IndZ(\cC)$ and $\coIndZ(\cC)$, we realize each $\mathbb{RP}^{2N-1}$ as the space of realizable distributions $\Delta_{\cC^\pm}$ for some partial concept class $\cC$.

The last step arises from an argument of Frick, Hosseini, and Vasileuski, which demonstrated that every finite free $\Z_2$-simplicial complex arises as the \emph{sign complex} for some partial sign matrix~\cite{frick2026signrank}. The sign complex is a simplicial complex associated with a partial sign matrix, and it parallels the role of the space of realizable distributions in our discussion of partial concept classes. The same proof in \cite{frick2026signrank} can be adapted to our setting:

\begin{lemma}\label{lem:universality_of_Delta_C}
    Every finite free $\bZ_2$-simplicial complex $\cK$ is isomorphic to $\Delta_{\cC^\pm}$ for a partial concept class $\cC$.
\end{lemma}

\begin{proof}
    A fixed-point-free involution $\nu$ on a finite simplicial complex $\cK$ pairs vertices of $\cK$ as $(v,\nu(v))$. Since $\cK$ is finite, we may index these pairs by $j\in [N]$. Likewise, the facets of $\cK$ also come in pairs $(F, \nu(F))$, which we can index by $i \in [M]$.
    
    Now define a partial concept class $\cC_\cK \subseteq \set{+1,-1,\star}^N$ of $M$ concepts given by
    \[
        c_i(j) =
        \begin{cases}
            +1 & \text{if $v_j \in F_i$}\\
            -1 & \text{if $\nu(v_j) \in F_i$}\\
            \star & \text{otherwise.}
        \end{cases}
    \]  
\end{proof}

Moreover, nice enough free $\bZ_2$-spaces can be given a free $\bZ_2$-simplicial structure through \emph{triangulation}. This is a classical result of S\"oren Illman.

\begin{theorem}[Illman's Theorem~\cite{illman_smooth_1978}]\label{thm:Illman}
    Every compact smooth free $\ZZ_2$-manifold admits a finite $\ZZ_2$-equivariant triangulation.
\end{theorem}

Combining \cref{lem:universality_of_Delta_C} and \cref{thm:Illman} allows us to import the crucial separating example $\mathbb{RP}^{2N-1}$.

\begin{proof}[Proof of \cref{cor:separation_LR_coInd}]
    As stated in \cite[page 101]{matousek2003borsuk} (also see \cite{frick2026signrank}),  the projective space $\mathbb{RP}^{2N-1}$ can be equipped with the fixed-point-free involution induced by multiplication by $i$ on $\mathbb{C}^N$. Moreover, this gives the separation
    \[
        \coIndZ(\mathbb{RP}^{2N-1}) = O(1) \qquad\text{and}\qquad
        \IndZ(\mathbb{RP}^{2N-1}) = \omega(1).
    \]
    By \cref{thm:Illman}, there is an antipodal homeomorphism between the free $\bZ_2$-space $\mathbb{RP}^{2N-1}$ and a simplicial complex $\cK_N$. Since $\IndZ$ and $\coIndZ$ are defined using topological properties of antipodal maps, it follows that $\cK_N$ is also a separating example:
    \[
        \coIndZ(\cK_N) = O(1) \qquad\text{and}\qquad
        \IndZ(\cK_N) = \omega(1).
    \]
    Applying \cref{lem:universality_of_Delta_C} yields a family of partial concept classes $\cC_N$ with the property that $\Delta_{\cC_{N}^\pm}$ is isomorphic to $\cK_N$. Using once more that $\IndZ$ and $\coIndZ$ are invariant under isomorphism completes the proof.
    \[
        \coIndZ(\cC_N) = O(1) \qquad\text{and}\qquad
        \IndZ(\cC_N) = \omega(1).
    \]
    By \cref{thm:index_vs_LR}, we have $\LR(\cC_N)=\omega(1)$. 
\end{proof}

\section{Composition properties for list replicability }\label{sec:properties_LR}

In this section, we prove some properties of list replicability under joins and concatenations.

\begin{lemma}[Concatenation]\label{lemma:LR_union}
Let $\cC_1,\cC_2\subseteq \{\pm1\}^\cX$ be concept classes over the same finite
domain $\cX$. Then, for every $\epsilon>0$,
\[
    \LR(\cC_1\cup \cC_2,\epsilon)
    \leq
    \LR(\cC_1,\epsilon)+\LR(\cC_2,\epsilon).
\]
Consequently,
\[
    \LR(\cC_1\cup \cC_2)\leq \LR(\cC_1)+\LR(\cC_2).
\]
\end{lemma}

\begin{proof}
Fix $\eps > 0$ and let $L_i:=\LR(\cC_i,\varepsilon)$ for $i=1,2$.

By \cref{thm:closed-cover-lr}, for each
$i\in\{1,2\}$ there is a closed cover $\mathcal F_i=\{F^i_h:h\in\{\pm1\}^\cX\}$
of $\Delta_{\cC_i}$ such that the overlap degree of $\mathcal F_i$ is at
most $L_i$, and
$F^i_h\subseteq B_\varepsilon(\sigma_h)$
   for every $h\in\{\pm1\}^\cX.$
Since $\Delta_{\cC_1\cup \cC_2}=\Delta_{\cC_1}\cup \Delta_{\cC_2},$
define, for every hypothesis $h\in\{\pm1\}^\cX$, the closed set
\[
    F_h:=F^1_h\cup F^2_h \subseteq \Delta_{\cC_1\cup \cC_2}.
\]
Then $\mathcal F:=\{F_h:h\in\{\pm1\}^\cX\}$ is a closed cover of $\Delta_{\cC_1\cup \cC_2}$ with overlap degree at most $L_1+L_2$ satisfying $F_h\subseteq B_\varepsilon(\sigma_h)$.

Again applying \cref{thm:closed-cover-lr} gives
\[
    \LR(\cC_1 \cup \cC_2,\epsilon)\leq L_1+L_2. 
\]
Taking the supremum over $\epsilon>0$ yields
\[
    \LR(\cC_1\cup \cC_2)\leq \LR(\cC_1)+\LR(\cC_2).
\]
\end{proof}

\begin{theorem}[Join of classes]\label{thm:lr_direct_product}
Let $\cC_1 \subseteq \Set{\pm 1}^{\cX_1}$ and $\cC_2 \subseteq \Set{\pm 1}^{\cX_2}$ be concept classes over finite disjoint domains. Define
\[
    \cC_1 * \cC_2
    \coloneqq
    \Set{c_1 \sqcup c_2 : c_1 \in \cC_1,\ c_2 \in \cC_2}
    \subseteq \Set{\pm 1}^{\cX_1 \sqcup \cX_2}.
\]
Then for every $\eps > 0$,
\[
    \LR(\cC_1 * \cC_2,\eps)
    \leq
    \LR(\cC_1,\eps/4)+\LR(\cC_2,\eps/4).
\]
Consequently,
\[
    \LR(\cC_1 * \cC_2)
    \leq
    \LR(\cC_1)+\LR(\cC_2).
\]
\end{theorem}

\begin{proof}
Fix $\epsilon > 0$ and let $ L_i \coloneqq \LR(\cC_i,\eps/4)$ for $i=1,2$. By the open-cover characterization of list replicability in \cref{thm:closed-cover-lr}, there are open covers
\[
    \cU = \Set{U_h : h \in \Set{\pm 1}^{\cX_1}}
    \qquad\text{and}\qquad
    \mathcal{V} = \Set{V_g : g \in \Set{\pm 1}^{\cX_2}}
\]
of $\Delta_{\cC_1}$ and $\Delta_{\cC_2}$, respectively, such that the overlap degree of $\cU$ is at most $L_1$, the overlap degree of $\mathcal{V}$ is at most $L_2$, and
$U_h \subseteq B_{\eps/3}(\sigma_h)$, $
    V_g \subseteq B_{\eps/3}(\sigma_g).$
In particular,
\begin{equation}\label{eq:local_good_sets}
    \mu_1 \in U_h \Rightarrow \loss_{\mu_1}(h)<\eps/3,
    \qquad
    \mu_2 \in V_g \Rightarrow \loss_{\mu_2}(g)<\eps/3.
\end{equation}
Choose partitions of unity subordinate to these covers, $\set{f^{(1)}_h}_{h\in\set{\pm1}^{\cX_1}}$, $    \set{f^{(2)}_g}_{g\in\set{\pm1}^{\cX_2}},
$
so that $\supp f^{(1)}_h \subseteq U_h$, $\supp f^{(2)}_g \subseteq V_g$, and
\[
    \sum_{h\in\Set{\pm1}^{\cX_1}} f^{(1)}_h \equiv 1,
    \qquad
    \sum_{g\in\Set{\pm1}^{\cX_2}} f^{(2)}_g \equiv 1.
\]
We construct an $(\eps,L_1+L_2)$-list-replicable learner for $\cC_1*\cC_2$.

Choose $0<\tau<\eps/3$ and a continuous cutoff function $r:[0,1]\to[0,1]$ such that
\[
    r(s)=0 \text{ for } s\le \tau/2,
    \qquad
    r(s)=1 \text{ for } s\ge \tau.
\]
Fix arbitrary default hypotheses $h_0\in\set{\pm1}^{\cX_1}$ and $g_0\in\set{\pm1}^{\cX_2}$.
Now choose arbitrary orderings $\set{\pm1}^{\cX_1}=\set{h_1,\ldots,h_M}$, $   \set{\pm1}^{\cX_2}=\Set{g_1,\ldots,g_N},$
where $M=2^{\abs{\cX_1}}$ and $N=2^{\abs{\cX_2}}$. The hypotheses $h_0$ and $g_0$ may repeat some $h_i$ or $g_j$; this causes no difficulty, since lists are sets of hypotheses and repetitions can only decrease their size.

Let $\mu\in\Delta_{\cC_1*\cC_2}$. Write $t \coloneqq \mu(\cX_1)$ and $1-t=\mu(\cX_2),$ and decompose
\[
    \mu = t\mu_1+(1-t)\mu_2,
\]
where $\mu_i\in\Delta_{\cC_i}$ is the conditional distribution on $\cX_i$ whenever the corresponding mass is nonzero.

Define probability vectors $p(\mu)$ on $\Set{0,1,\ldots,M}$ and $q(\mu)$ on $\Set{0,1,\ldots,N}$ by
\[
    p_0(\mu) \coloneqq 1-r(t),
    \qquad
    p_i(\mu) \coloneqq r(t) f^{(1)}_{h_i}(\mu_1) \quad (1\le i\le M),
\]
and
\[
    q_0(\mu) \coloneqq 1-r(1-t),
    \qquad
    q_j(\mu) \coloneqq r(1-t) f^{(2)}_{g_j}(\mu_2) \quad (1\le j\le N).
\]
Let
\[
    P_a(\mu) \coloneqq \sum_{i=0}^{a}p_i(\mu),
    \qquad
    Q_b(\mu) \coloneqq \sum_{j=0}^{b}q_j(\mu),
\]
with $P_{-1}(\mu)=Q_{-1}(\mu)=0$. Define intervals
\[
    I_a(\mu) \coloneqq [P_{a-1}(\mu),P_a(\mu)]
    \qquad (0\le a\le M),
\]
and
\[
    J_b(\mu) \coloneqq [Q_{b-1}(\mu),Q_b(\mu)]
    \qquad (0\le b\le N).
\]
Finally define
\[
    w_{a,b}(\mu)
    \coloneqq
    \abs{I_a(\mu)\cap J_b(\mu)}.
\]
In particular, every $w_{a,b}$ is continuous on $\Delta_{C_1* C_2}$ and
\[
    \sum_{a=0}^{M}\sum_{b=0}^{N} w_{a,b}(\mu)=1.
\]
Define the good list for $\mu$ by
\[
    \operatorname{List}(\mu)
    \coloneqq
    \Set{h_a\sqcup g_b : w_{a,b}(\mu)>0},
\]
where $h_a$ means the default $h_0$ when $a=0$, and similarly $g_b$ means $g_0$ when $b=0$.

\begin{claim}\label{clm:product_list_good}
For every $\mu\in\Delta_{\cC_1*\cC_2}$,
\[
    \abs{\operatorname{List}(\mu)}\le L_1+L_2,
\]
and every hypothesis in $\operatorname{List}(\mu)$ has $\mu$-loss less than $\eps$.
\end{claim}

\begin{proof}
Let
\[
    m(\mu)\coloneqq \abs{\Set{a:p_a(\mu)>0}},
    \qquad
    n(\mu)\coloneqq \abs{\Set{b:q_b(\mu)>0}}.
\]
For two interval partitions of $[0,1]$ with $m(\mu)$ and $n(\mu)$ positive-length intervals, the number of positive-length intersections is at most
\[
    m(\mu)+n(\mu)-1.
\]
Indeed, sweeping from left to right, the active pair changes only when one crosses an endpoint of one of the two partitions.

If $t\ge \tau$, then $p_0(\mu)=0$ and at most $L_1$ of the non-default $p_i(\mu)$ are positive. If $t\le \tau/2$, then $p_0(\mu)=1$. If $\tau/2<t<\tau$, then at most $L_1+1$ of the $p_i(\mu)$'s are positive. The same statements hold for $q(\mu)$ with $1-t$ in place of $t$. Since $\tau<1/3$, the two transition regimes $\tau/2<t<\tau$ and $\tau/2<1-t<\tau$ cannot occur simultaneously. Hence
\[
    \abs{\operatorname{List}(\mu)}
    \le m(\mu)+n(\mu)-1
    \le L_1+L_2.
\]

Now take any glued hypothesis $h_a \sqcup g_b$ in $\operatorname{List}(\mu)$. Then $p_a(\mu)>0$ and $q_b(\mu)>0$. If $a\neq 0$, then $f^{(1)}_{h_a}(\mu_1)>0$, and therefore $\mu_1\in U_{h_a}$; by \eqref{eq:local_good_sets}, $\loss_{\mu_1}(h_a)<\eps/3$. If $a=0$, then $p_0(\mu)>0$, so $r(t)<1$, and hence $t<\tau<\eps/3$.

Similarly, if $b\neq 0$, then $\loss_{\mu_2}(g_b)<\eps/3$, while if $b=0$, then $1-t<\tau$. The case $a=b=0$ is impossible, since it would imply $t<\tau$ and $1-t<\tau<\eps/3$. Therefore
\[
    \loss_{\mu}(h_a\sqcup g_b)
    =
    t\loss_{\mu_1}(h_a)+(1-t)\loss_{\mu_2}(g_b)
    <
    \eps/3+\eps/3
    <\eps.
\]
This proves the claim.
\end{proof}

We now define the learning rule.

\begin{algorithm}[H]
\caption{\label{alg:product_zip_learner}The join learning rule $\cA$}
\begin{algorithmic}[1]
    \STATE Sample $S=(S_1,\ldots,S_n)\sim \mu^n$.
    \STATE Construct the empirical distribution $\widehat{\mu}$ on $(\cX_1\sqcup\cX_2)\times\Set{\pm1}$.
    \STATE Compute the weights $w_{a,b}(\widehat{\mu})$ for all $0\le a\le M$ and $0\le b\le N$.
    \STATE Output $h_a\sqcup g_b$ with probability $w_{a,b}(\widehat{\mu})$.
\end{algorithmic}
\end{algorithm}

It remains to show that for sufficiently large $n$, the output belongs to $\operatorname{List}(\mu)$ with probability at least $1-\delta$, uniformly over $\mu$.

For a fixed $\mu\in\Delta_{\cC_1*\cC_2}$ define
\[
    F_\mu(\nu)
    \coloneqq
    \sum_{\substack{0\le a\le M,\ 0\le b\le N:\\ h_a\sqcup g_b\notin \operatorname{List}(\mu)}}
    w_{a,b}(\nu).
\]
This is the probability that the algorithm, when run with empirical distribution $\nu$, outputs a hypothesis outside $\operatorname{List}(\mu)$. Since the sum is finite and the functions $w_{a,b}$ are continuous, $F_\mu$ is continuous. Moreover,
\[
    F_\mu(\mu)=0.
\]
The compactness of $\Delta_{\cC_1*\cC_2}$ implies the finite family of functions $w_{a,b}$ is uniformly equicontinuous. More explicitly, choose $\rho>0$ so that each $w_{a,b}$ changes by at most $\delta/(2R)$ on $\rho$-balls, where $R=(M+1)(N+1)$. Thus this $\rho$ is independent of $\mu$, and whenever $d_{\TV}(\nu,\mu)<\rho$, we have
\[
    F_\mu(\nu)<\delta/2.
\]
Since $\cX_1\sqcup\cX_2$ is finite, standard uniform convergence of empirical distributions gives an $n_0=n_0(\rho,\delta,\cX_1,\cX_2)$ such that for every realizable $\mu$ and every $n\ge n_0$,
\[
    \Pr_{S\sim\mu^n}\left[d_{\TV}(\widehat{\mu},\mu)<\rho\right]
    \ge 1-\delta/2.
\]
Consequently, for every $\mu\in\Delta_{\cC_1*\cC_2}$ and $n\ge n_0$,
\begin{align*}
    \Pr_{S\sim\mu^n}\left[\cA(S)\notin\operatorname{List}(\mu)\right]
    &\le
    \Pr\left[d_{\TV}(\widehat{\mu},\mu)\ge\rho\right]
    +
    \E\left[F_\mu(\widehat{\mu})\mathbf{1}_{\{d_{\TV}(\widehat{\mu},\mu)<\rho\}}\right]  \\
    &\le \delta/2+\delta/2
    =\delta.
\end{align*}
Together with \cref{clm:product_list_good}, this shows that $\cA$ is an $(\eps,L_1+L_2)$-list-replicable learner for $\cC_1 * \cC_2$. Therefore
\[
    \LR(\cC_1 * \cC_2,\eps)
    \le L_1+L_2
    =\LR(\cC_1,\eps/4)+\LR(\cC_2,\eps/4).
\]
Taking the supremum over $\eps>0$ proves the final assertion.
\end{proof}

\bibliographystyle{alphaurl}
\bibliography{refs_07_21}

@article{russo2013eluder,
  title={Eluder dimension and the sample complexity of optimistic exploration},
  author={Russo, Daniel and Van Roy, Benjamin},
  journal={Advances in Neural Information Processing Systems},
  volume={26},
  year={2013}
}

@book{MahadevPeled1995,
    AUTHOR = {Mahadev, N. V. R. and Peled, U. N.},
     TITLE = {Threshold graphs and related topics},
    SERIES = {Annals of Discrete Mathematics},
    VOLUME = {56},
 PUBLISHER = {North-Holland Publishing Co., Amsterdam},
      YEAR = {1995},
     PAGES = {xiv+543},
      ISBN = {0-444-89287-7}
}

@article {HeggernesKratsch2007Certifying,
    AUTHOR = {Heggernes, Pinar and Kratsch, Dieter},
     TITLE = {Linear-time certifying recognition algorithms and forbidden
              induced subgraphs},
   JOURNAL = {Nordic J. Comput.},
  FJOURNAL = {Nordic Journal of Computing},
    VOLUME = {14},
      YEAR = {2007},
    NUMBER = {1-2},
     PAGES = {87--108},
      ISSN = {1236-6064}
}

@article{lovasz1978kneser,
  title={Kneser's conjecture, chromatic number, and homotopy},
  author={Lov{\'a}sz, L{\'a}szl{\'o}},
  journal={Journal of Combinatorial Theory, Series A},
  volume={25},
  number={3},
  pages={319--324},
  year={1978},
  publisher={Elsevier}
}

@article{matousek2002topological,
  title={Topological lower bounds for the chromatic number: A hierarchy},
  author={Matousek, Jiri and Ziegler, G{\"u}nter M},
  journal={arXiv preprint math/0208072},
  year={2002}
}

@article{csorba2004box,
  title={Box complexes, neighborhood complexes, and the chromatic number},
  author={Csorba, P{\'e}ter and Lange, Carsten and Schurr, Ingo and Wassmer, Arnold},
  journal={Journal of Combinatorial Theory, Series A},
  volume={108},
  number={1},
  pages={159--168},
  year={2004},
  publisher={Elsevier}
}

@article{simonyi2006local,
  title={Local chromatic number, {K}y {F}an's theorem, and circular colorings},
  author={Simonyi, G{\'a}bor and Tardos, G{\'a}bor},
  journal={Combinatorica},
  volume={26},
  number={5},
  pages={587--626},
  year={2006},
  publisher={Springer}
}

@article{simonyi2009local,
  title={Local chromatic number and distinguishing the strength of topological obstructions},
  author={Simonyi, G{\'a}bor and Tardos, G{\'a}bor and Vre{\'c}ica, Sini{\v{s}}a},
  journal={Transactions of the American Mathematical Society},
  volume={361},
  number={2},
  pages={889--908},
  year={2009}
}

@book{matousek2003borsuk,
  title={Using the {B}orsuk-{U}lam Theorem: Lectures on Topological Methods in Combinatorics and Geometry},
  author={Matoušek, Jiří},
  year={2003},
  publisher={Springer},
  address={Berlin/Heidelberg},
  series={Universitext},
  isbn={978-3-540-00362-5}
}

@book{rudin1987real,
  title={Real and complex analysis},
  author={Rudin, Walter},
  year={1987},
  publisher={McGraw-Hill, Inc.}
}

@misc{frick2026signrank,
      title={A $\mathbb{Z}_2$-Topological Framework for Sign-rank Lower Bounds}, 
      author={Florian Frick and Kaave Hosseini and Aliaksei Vasileuski},
      year={2026},
      eprint={2604.01510},
      archivePrefix={arXiv},
      primaryClass={math.CO},
 
}

@inproceedings{hatami2022lower,
  author =	{Hatami, Hamed and Hatami, Pooya and Pires, William and Tao, Ran and Zhao, Rosie},
  title =	{{Lower Bound Methods for Sign-Rank and Their Limitations}},
  booktitle =	{Approximation, Randomization, and Combinatorial Optimization. Algorithms and Techniques (APPROX/RANDOM 2022)},
  pages =	{22:1--22:24},
  series =	{Leibniz International Proceedings in Informatics (LIPIcs)},
  ISBN =	{978-3-95977-249-5},
  ISSN =	{1868-8969},
  year =	{2022},
  volume =	{245},
  publisher =	{Schloss Dagstuhl -- Leibniz-Zentrum f{\"u}r Informatik},
}

@article{rubinstein2022unlabelled,
  title={Unlabelled sample compression schemes for intersection-closed classes and extremal classes},
  author={Rubinstein, Joachim and Rubinstein, Benjamin},
  journal={Advances in Neural Information Processing Systems},
  volume={35},
  pages={13078--13090},
  year={2022}
}

@INPROCEEDINGS{BLM20,
  author={Bun, Mark and Livni, Roi and Moran, Shay},
  booktitle={IEEE 61st Annual Symposium on Foundations of Computer Science (FOCS)}, 
  title={An Equivalence Between Private Classification and Online Prediction}, 
  year={2020},
  pages={389-402},
}

@article{malliaris2022unstable,
  title={The unstable formula theorem revisited via algorithms},
  author={Malliaris, Maryanthe and Moran, Shay},
  journal={arXiv preprint arXiv:2212.05050},
  year={2022}
}

@inproceedings{bun2023stability,
  title={Stability is stable: Connections between replicability, privacy, and adaptive generalization},
  author={Bun, Mark and Gaboardi, Marco and Hopkins, Max and Impagliazzo, Russell and Lei, Rex and Pitassi, Toniann and Sivakumar, Satchit and Sorrell, Jessica},
  booktitle={Proceedings of the 55th Annual ACM Symposium on Theory of Computing},
  series = {STOC 2023}, 
  year = {2023},
  pages={520--527}
}

@article{karbasi2023replicability,
  title={Replicability in reinforcement learning},
  author={Karbasi, Amin and Velegkas, Grigoris and Yang, Lin and Zhou, Felix},
  journal={Advances in Neural Information Processing Systems},
  volume={36},
  pages={74702--74735},
  year={2023}
}

@inproceedings{esfandiari2023replicable,
  title={Replicable Bandits},
  author={Hossein Esfandiari and Alkis Kalavasis and Amin Karbasi and Andreas Krause and Vahab Mirrokni and Grigoris Velegkas},
  booktitle={The Eleventh International Conference on Learning Representations },
  year={2023},
 
}

@inproceedings{Esfandiarietal23,
  author = {Esfandiari, Hossein and Karbasi, Amin and Mirrokni, Vahab and Velegkas, Grigoris and Zhou, Felix},
  booktitle = {Advances in Neural Information Processing Systems},
  pages = {39277--39320},
  publisher = {Curran Associates, Inc.},
  title = {Replicable Clustering},
  volume = {36},
  year = {2023}
}

@article{moran2023bayesian,
  title={The {B}ayesian stability zoo},
  author={Moran, Shay and Schefler, Hilla and Shafer, Jonathan},
  journal={Advances in Neural Information Processing Systems},
  volume={36},
  pages={61725--61746},
  year={2023}
}

@article{li2022understanding,
  title={Understanding the eluder dimension},
  author={Li, Gene and Kamath, Pritish and Foster, Dylan J and Srebro, Nati},
  journal={Advances in Neural Information Processing Systems},
  volume={35},
  pages={23737--23750},
  year={2022}
}

@article{atminas2022classes,
  title={Classes of graphs without star forests and related graphs},
  author={Atminas, Aistis},
  journal={Discrete Mathematics},
  volume={345},
  number={12},
  pages={113089},
  year={2022},
  publisher={Elsevier}
}

@article{hanneke2015minimax,
  title={Minimax analysis of active learning},
  author={Hanneke, Steve and Yang, Liu},
  journal={The Journal of Machine Learning Research},
  volume={16},
  number={1},
  pages={3487--3602},
  year={2015},
  publisher={JMLR. org}
}

@inproceedings{eaton2024replicable,
  author = {Eaton, Eric and Hussing, Marcel and Kearns, Michael and Sorrell, Jessica},
  title = {Replicable reinforcement learning},
  year = {2023},
  publisher = {Curran Associates Inc.},
  booktitle = {Proceedings of the 37th International Conference on Neural Information Processing Systems},
  articleno = {667},
  numpages = {14},
  location = {New Orleans, LA, USA},
  series = {NeurIPS '23}
}

@inproceedings{kalavasis2024replicable,
  author = {Kalavasis, Alkis and Karbasi, Amin and Larsen, Kasper Green and Velegkas, Grigoris and Zhou, Felix},
  title = {Replicable learning of large-margin halfspaces},
  year = {2024},
  publisher = {JMLR.org},
  booktitle = {Proceedings of the 41st International Conference on Machine Learning},
  articleno = {919},
  numpages = {18},
  location = {Vienna, Austria},
  series = {ICML'24}
}

@article{illman_smooth_1978,
  title = {Smooth equivariant triangulations of {G}-manifolds for {G} a finite group},
  volume = {233},
  issn = {1432-1807},
  number = {3},
  journal = {Mathematische Annalen},
  author = {Illman, Sören},
  month = oct,
  year = {1978},
  pages = {199--220},
}

@inproceedings{alon1985geometrical,
  author    = {Noga Alon and
               Peter Frankl and
               Vojtech R{\"{o}}dl},
  title     = {Geometrical Realization of Set Systems and Probabilistic Communication Complexity},
  booktitle = {26th Annual Symposium on Foundations of Computer Science,  FOCS 1985},
  pages     = {277--280},
  publisher = {{IEEE} Computer Society},
  year      = {1985},
}

@inproceedings{moran2016labeled,
  title={Labeled compression schemes for extremal classes},
  author={Moran, Shay and Warmuth, Manfred K},
  booktitle={International Conference on Algorithmic Learning Theory},
  pages={34--49},
  year={2016},
  organization={Springer}
}

@article{paturi1986probabilistic,
  title={Probabilistic communication complexity},
  author={Paturi, Ramamohan and Simon, Janos},
  journal={Journal of Computer and System Sciences},
  volume={33},
  number={1},
  pages={106--123},
  year={1986},
  publisher={Elsevier}
}

@incollection {MR0200942,
    AUTHOR = {Thom, Ren\'{e}},
     TITLE = {Sur l'homologie des vari\'{e}t\'{e}s alg\'{e}briques r\'{e}elles},
 BOOKTITLE = {Differential and {C}ombinatorial {T}opology ({A} {S}ymposium
              in {H}onor of {M}arston {M}orse)},
     PAGES = {255--265},
 PUBLISHER = {Princeton Univ. Press, Princeton, N.J.},
      YEAR = {1965},
   MRCLASS = {57.31 (57.50)},
  MRNUMBER = {0200942},
MRREVIEWER = {R. Bott},
}

@article {MR226281,
    AUTHOR = {Warren, Hugh E.},
     TITLE = {Lower bounds for approximation by nonlinear manifolds},
   JOURNAL = {Trans. Amer. Math. Soc.},
  FJOURNAL = {Transactions of the American Mathematical Society},
    VOLUME = {133},
      YEAR = {1968},
     PAGES = {167--178},
      ISSN = {0002-9947},
   MRCLASS = {41.60},

}

@article {MR161339,
    AUTHOR = {Milnor, J.},
     TITLE = {On the {B}etti numbers of real varieties},
   JOURNAL = {Proc. Amer. Math. Soc.},
  FJOURNAL = {Proceedings of the American Mathematical Society},
    VOLUME = {15},
      YEAR = {1964},
     PAGES = {275--280},
      ISSN = {0002-9939},
   MRCLASS = {55.30 (57.50)},

}

@article {MR1964645,
    AUTHOR = {Forster, J\"{u}rgen},
     TITLE = {A linear lower bound on the unbounded error probabilistic
              communication complexity},
      NOTE = {Special issue on complexity, 2001 (Chicago, IL)},
   JOURNAL = {J. Comput. System Sci.},
  FJOURNAL = {Journal of Computer and System Sciences},
    VOLUME = {65},
      YEAR = {2002},
    NUMBER = {4},
     PAGES = {612--625},

}

@inproceedings{alon2016sign,
    title={Sign rank versus {VC} dimension},
    author={Alon, Noga and Moran, Shay and Yehudayoff, Amir},
    booktitle={Conference on Learning Theory},
    pages={47--80},
    year={2016},
    organization={PMLR}
}

@article{Alon_22_private_and_online,
    AUTHOR = {Alon, Noga and Bun, Mark and Livni, Roi and Malliaris, Maryanthe and Moran, Shay},
    TITLE = {Private and online learnability are equivalent},
    JOURNAL = {J. ACM},
    FJOURNAL = {Journal of the ACM},
    VOLUME = {69},
    YEAR = {2022},
    NUMBER = {4},
    PAGES = {Art. 28, 34}
}

@inproceedings{hanneke2024star,
  title={The star number and eluder dimension: Elementary observations about the dimensions of disagreement},
  author={Hanneke, Steve},
  booktitle={The Thirty Seventh Annual Conference on Learning Theory},
  pages={2308--2359},
  year={2024},
  organization={PMLR}
}

@inproceedings{BGHH2025stabilitylistreplicabilityagnosticlearners,
    author = {Ari Blondal and Shan Gao and Hamed Hatami and Pooya Hatami},
    title = {Stability and List-Replicability for Agnostic Learners},
    year = {2025},
    booktitle = {Conference on Learning Theory},
    pages = {380-–400},
    organization={PMLR}
}

@inproceedings{blondal2026borsuk,
  title={Borsuk-{U}lam and replicable learning of large-margin halfspaces},
  author={Blondal, Ari and Hatami, Hamed and Hatami, Pooya and Lalov, Chavdar and Tretiak, Sivan},
  booktitle={Proceedings of the 58th Annual ACM Symposium on Theory of Computing},
  pages={529--540},
  year={2026}
}

@InProceedings{blondal2026simplicial,
    author =	{Blondal, Ari and Hatami, Hamed and Hatami, Pooya and Lalov, Chavdar and Tretiak, Sivan},
    title =	{{Simplicial Covering Dimension of Extremal Concept Classes}},
    booktitle =	{17th Innovations in Theoretical Computer Science Conference (ITCS 2026)},
    pages =	{22:1--22:24},
    series =	{Leibniz International Proceedings in Informatics (LIPIcs)},
    ISBN =	{978-3-95977-410-9},
    ISSN =	{1868-8969},
    year =	{2026},
    volume =	{362},
    editor =	{Saraf, Shubhangi},
    publisher =	{Schloss Dagstuhl -- Leibniz-Zentrum f{\"u}r Informatik},
    address =	{Dagstuhl, Germany},
}

@inproceedings{blondal2026tightlistreplicabilitybounds,
    title={Tight list replicability bounds via a novel sphere covering theorem}, 
    author={Ari Blondal and Hamed Hatami and Pooya Hatami and Chavdar Lalov and Sivan Tretiak},
  booktitle = 	 {Proceedings of Thirty Ninth Conference on Learning Theory},
  pages = 	 {791--807},
  year = 	 {2026},
  editor = 	 {Hanneke, Steve and Lattimore, Tor},
  volume = 	 {336},
  series = 	 {Proceedings of Machine Learning Research},
  month = 	 {29 Jun--03 Jul},
  publisher =    {PMLR},
}

@inproceedings{chase2023replicabilitystabilitylearning,
    author = {Chase, Zachary and Moran, Shay and Yehudayoff, Amir },
    booktitle = {IEEE 64th Annual Symposium on Foundations of Computer Science (FOCS) },
    title = {{Stability and Replicability in Learning}},
    year = {2023},
    volume = {},
    ISSN = {},
    pages = {2430-2439},
    publisher = {IEEE Computer Society},
}

@inproceedings{localborsukulam,
    author = {Chase, Zachary and Chornomaz, Bogdan and Moran, Shay and Yehudayoff, Amir},
    title = {Local {B}orsuk-{U}lam, Stability, and Replicability},
    year = {2024},
    isbn = {9798400703836},
    publisher = {Association for Computing Machinery},
    address = {New York, NY, USA},
    booktitle = {Proceedings of the 56th Annual ACM Symposium on Theory of Computing},
    pages = {1769–1780},
    numpages = {12},
    location = {Vancouver, BC, Canada},
    series = {STOC 2024}
}

@InProceedings{chornomaz2025spherical,
    title = 	 {Spherical Dimension},
    author =       {Chornomaz, Bogdan and Moran, Shay and Waknine, Tom},
    booktitle = 	 {Proceedings of Thirty Eighth Conference on Learning Theory},
    pages = 	 {1259--1313},
    year = 	 {2025},
    editor = 	 {Haghtalab, Nika and Moitra, Ankur},
    volume = 	 {291},
    series = 	 {Proceedings of Machine Learning Research},
    month = 	 {30 Jun--04 Jul},
    publisher =    {PMLR},
 
}

@inproceedings{dixon2023listandcertificate,
    author = {Dixon, Peter and Pavan, A. and Vander Woude, Jason and Vinodchandran, N. V.},
    title = {List and certificate complexities in replicable learning},
    year = {2023},
    publisher = {Curran Associates Inc.},
    address = {Red Hook, NY, USA},
    booktitle = {Proceedings of the 37th International Conference on Neural Information Processing Systems},
    articleno = {1342},
    numpages = {23},
    location = {New Orleans, LA, USA},
    series = {NeurIPS '23}
}

@inproceedings{chase2024dual,
  title={Dual {VC} dimension obstructs sample compression by embeddings},
  author={Chase, Zachary and Chornomaz, Bogdan and Hanneke, Steve and Moran, Shay and Yehudayoff, Amir},
  booktitle={The Thirty Seventh Annual Conference on Learning Theory},
  pages={923--946},
  year={2024},
  organization={PMLR}
}

@inproceedings{kalavasis2023statistical,
    author = {Kalavasis, Alkis and Karbasi, Amin and Moran, Shay and Velegkas, Grigoris},
    title = {Statistical indistinguishability of learning algorithms},
    year = {2023},
    publisher = {JMLR.org},
    booktitle = {Proceedings of the 40th International Conference on Machine Learning},
    articleno = {636},
    numpages = {37},
    location = {Honolulu, Hawaii, USA},
    series = {ICML'23}
}

@article{helmbold1990learning,
  title={Learning nested differences of intersection-closed concept classes},
  author={Helmbold, David and Sloan, Robert and Warmuth, Manfred K},
  journal={Machine Learning},
  volume={5},
  number={2},
  pages={165--196},
  year={1990},
  publisher={Springer}
}

@article{haussler1994predicting,
  title={Predicting $\{$0, 1$\}$-functions on randomly drawn points},
  author={Haussler, David and Littlestone, Nick and Warmuth, Manfred K},
  journal={Information and Computation},
  volume={115},
  number={2},
  pages={248--292},
  year={1994},
  publisher={Elsevier}
}

@article{floyd1995sample,
  title={Sample compression, learnability, and the {V}apnik-{C}hervonenkis dimension},
  author={Floyd, Sally and Warmuth, Manfred},
  journal={Machine Learning},
  volume={21},
  number={3},
  pages={269--304},
  year={1995},
  publisher={Springer}
}

@article{bendavid1998self,
  title={Self-directed learning and its relation to the {VC}-dimension and to teacher-directed learning},
  author={Ben-David, Shai and Eiron, Nadav},
  journal={Machine Learning},
  volume={33},
  number={1},
  pages={87--104},
  year={1998},
  publisher={Springer}
}

@inproceedings{kuhlmann1999teaching,
  title={On teaching and learning intersection-closed concept classes},
  author={Kuhlmann, Christian},
  booktitle={European Conference on Computational Learning Theory},
  pages={168--182},
  year={1999},
  organization={Springer}
}

@article{dalmau2003learnability,
  title={Learnability of quantified formulas},
  author={Dalmau, Victor and Jeavons, Peter},
  journal={Theoretical Computer Science},
  volume={306},
  number={1-3},
  pages={485--511},
  year={2003},
  publisher={Elsevier}
}

@article{auer2007new,
  title={A new {PAC} bound for intersection-closed concept classes},
  author={Auer, Peter and Ortner, Ronald},
  journal={Machine Learning},
  volume={66},
  number={2},
  pages={151--163},
  year={2007},
  publisher={Springer}
}

@article{darnstadt2015optimal,
  title={The optimal {PAC} bound for intersection-closed concept classes},
  author={Darnst{\"a}dt, Malte},
  journal={Information Processing Letters},
  volume={115},
  number={4},
  pages={458--461},
  year={2015},
  publisher={Elsevier}
}

@article{hanneke2016refined,
  title={Refined error bounds for several learning algorithms},
  author={Hanneke, Steve},
  journal={Journal of Machine Learning Research},
  volume={17},
  number={135},
  pages={1--55},
  year={2016}
}

@inproceedings{blum2021robust,
  title={Robust learning under clean-label attack},
  author={Blum, Avrim and Hanneke, Steve and Qian, Jian and Shao, Han},
  booktitle={Conference on Learning Theory},
  pages={591--634},
  year={2021},
  organization={PMLR}
}

@article{chalopin2022unlabeled,
  title={Unlabeled sample compression schemes and corner peelings for ample and maximum classes},
  author={Chalopin, J{\'e}r{\'e}mie and Chepoi, Victor and Moran, Shay and Warmuth, Manfred K},
  journal={Journal of Computer and System Sciences},
  volume={127},
  pages={1--28},
  year={2022},
  publisher={Elsevier}
}

\end{document}